\documentclass{article}

\usepackage{PRIMEarxiv}

\usepackage[utf8]{inputenc} % allow utf-8 input
\usepackage[T1]{fontenc}    % use 8-bit T1 fonts
\usepackage{hyperref}       % hyperlinks
\usepackage{url}            % simple URL typesetting
\usepackage{booktabs}       % professional-quality tables
\usepackage{amsfonts}       % blackboard math symbols
\usepackage{nicefrac}       % compact symbols for 1/2, etc.
\usepackage{microtype}      % microtypography
\usepackage{lipsum}
\usepackage{fancyhdr}       % header
\usepackage{graphicx}       % graphics
\graphicspath{{media/}}     % organize your images and other figures under media/ folder

\usepackage{amsmath}
\usepackage{amssymb}
\usepackage{mathtools}
\usepackage{amsthm}
\usepackage{subcaption}
\usepackage{booktabs}
\usepackage{graphicx}
\usepackage{algorithm}
\usepackage{wrapfig}

\usepackage{algpseudocode}
\usepackage{float}
\usepackage{amsthm}
\newtheorem{theorem}{Theorem}
\newtheorem{assumption}{Assumption}
\usepackage[numbers]{natbib}

\title{ISEP: Implicit Support Expansion for Offline Reinforcement Learning via Stochastic Policy Optimization
}

\author{
  Yifei Chen \quad Shaoqin Zhu \quad Xiaoqiang Ji\thanks{Corresponding author.} \\
  The Chinese University of Hong Kong, Shenzhen \\
  Longgang, Shenzhen, Guangdong, China \\
  \texttt{carolinechen0908@gmail.com} \\
  \texttt{shaoqinzhu@link.cuhk.edu.cn}, \texttt{jixiaoqiang@cuhk.edu.cn}
}

\begin{document}
\maketitle

\begin{abstract}
Offline reinforcement learning methods typically enforce strict constraints to ensure safety; yet this rigidity often prevents the discovery of optimal behaviors outside the immediate support of the behavior policy. To address this, we propose Implicit Support Expansion via stochastic Policy optimization (ISEP), which leverages a value function interpolated between in-distribution data and policy samples to implicitly expand the feasible action support. This mechanism "densifies" high-reward regions, creating a navigable path for policy improvement while theoretically guaranteeing bounded value error. However, optimizing against this expanded support creates a multimodal landscape where standard deterministic averaging leads to mode collapse and invalid actions. ISEP mitigates this via a stochastic action selection strategy, optimizing the policy by stochastically alternating between conservative cloning and optimistic expansion signals. We instantiate this framework as ISEP-FM using Conditional Flow Matching utilizing classifier-free guidance to effectively capture the interpolated value signal. 
\end{abstract}

% keywords can be removed
\keywords{Offline Reinforcement learning}

\begin{figure}[ht]
\centering
\includegraphics[width=0.7\columnwidth]{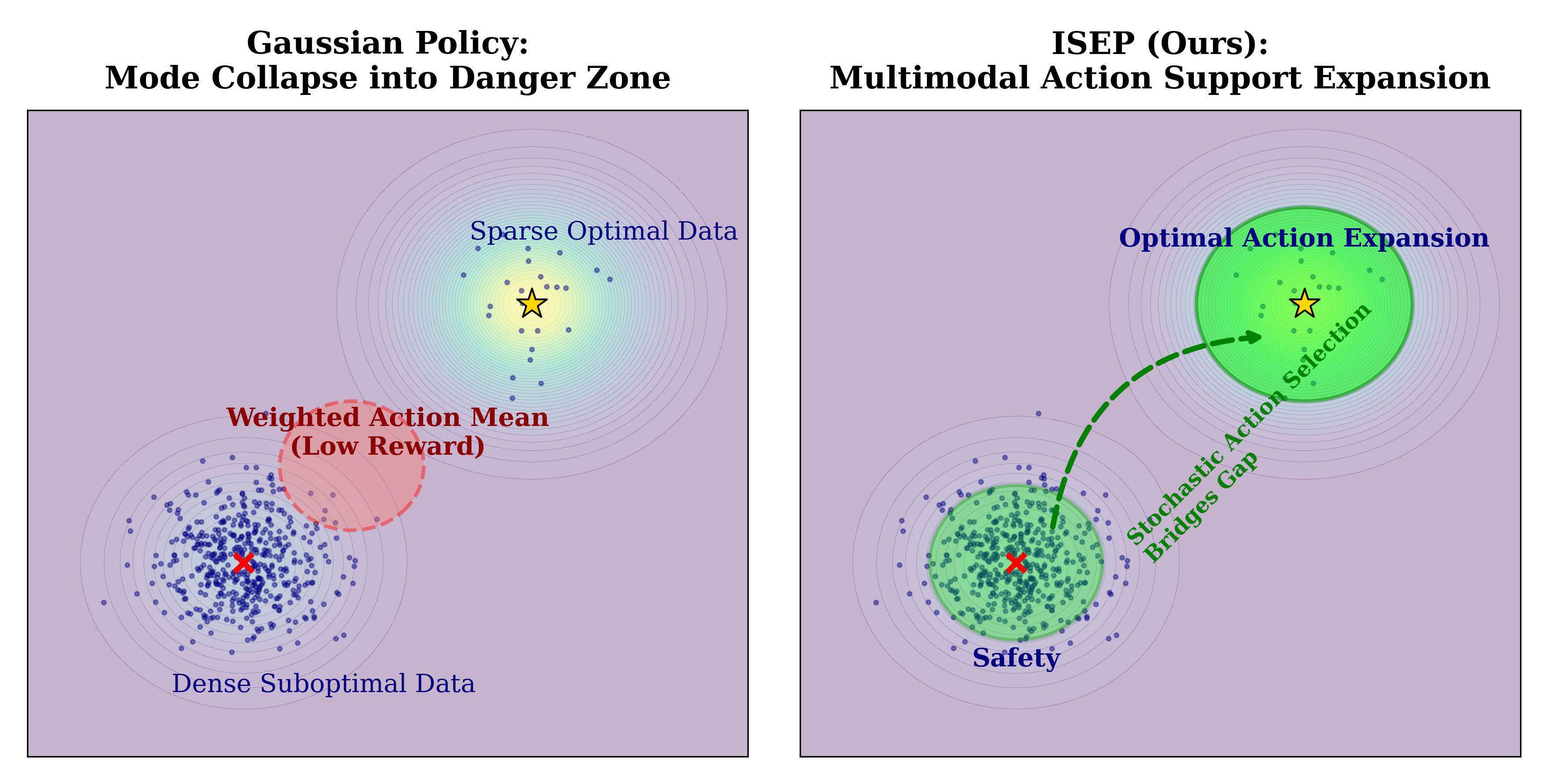}
\caption{
\textbf{Conceptual Illustration of Support Expansion on Disjoint Manifolds.} The environment consists of two safe, high-reward "islands" (radius $r=1$) surrounded by a low-reward danger background. The dataset is heavily skewed towards the suboptimal island (bottom-left).\textbf{(Left) Gaussian Failure:} A unimodal Gaussian policy attempts to cover both the dense suboptimal data and the sparse optimal data. This results in the distribution stretching across the background, placing significant probability mass in the danger zone (the "void" between islands).\textbf{(Right) ISEP (Ours):} By utilizing Flow Matching with a stochastic action selection strategy, ISEP learns a disjoint, multimodal policy. It maintains a mode anchored to the dense data while successfully expanding a distinct mode on the optimal island, bridging the value gap without crossing the dangerous background.
}
\label{fig:concept_fig}
\end{figure}
\section{Introduction}
Offline reinforcement learning (RL) aims to learn effective policies from static datasets without online interaction, enabling application in domains where exploration is costly or unsafe, such as robotics, healthcare, and autonomous systems~\cite{fujimoto2019offpolicydeepreinforcementlearning, offlinerl_survey, offlinerl_survey2}. A fundamental challenge in this setting is balancing between conservatism—staying within the dataset's support to ensure valid value estimation—and optimism—generalizing beyond suboptimal data to discover improved behaviors. Prior methods typically resolve this by enforcing strict support constraints: they either limit updates to in-sample transitions~\cite{iql,eql, sql}, restrict actions to the behavior policy’s support~\cite{bcq, td3bc, bear, spot}, or penalize overestimated critic values~\cite{brac, cql}. While these constraints ensure safety, they often hinder the discovery of optimal policies by suppressing high-reward behaviors that might be sparsely represented in the dataset.

To overcome this limitation, we propose \textbf{ISEP (Implicit Support Expansion via stochastic Policy optimization)}. Unlike prior methods that treat the dataset distribution as a rigid boundary, ISEP implicitly expands the valid support of the value function to encompass high-value actions that are valid but sparsely represented. The core of our approach is a hybrid objective that interpolates between in-distribution regression and exploratory queries on policy-generated samples. This mechanism effectively "densifies" the support in high-reward regions, creating a safe bridge for the policy to improve. This trade-off is governed by a controllable scalar parameter; we theoretically show that this expansion guaranties bounded value estimates, enabling controlled extrapolation without the risk of catastrophic divergence.

However, expanding the support introduces a new challenge for policy optimization. The resulting objective is inherently multimodal: the policy receives conflicting signals to both clone the conservative data and explore the expanded high-value regions. As illustrated in Figure~\ref{fig:concept_fig}, in the non-convex landscapes typical of RL, a standard deterministic interpolation between these actions leads to "mode collapse"—averaging a safe action and an optimistic action often results in a low-value action off the data manifold. To address this, ISEP employs a \textbf{Stochastic Action Selection} strategy. Instead of minimizing a weighted mean objective, we stochastically alternate between in-distribution actions and expanded actions, preserving the distinct modes of valid behavior.

This stochastic formulation necessitates a policy architecture capable of representing complex, multimodal distributions, which standard unimodal Gaussian policies fail to capture. Consequently, we instantiate our framework as ISEP-FM using classifier-free guidance Flow Matching~\cite{cfgrl}, ensuring the flow effectively captures the hybrid learning signal.

\section{Related Work}
To mitigate extrapolation error in offline RL, a variety of algorithms have been proposed. One prominent direction utilizes in-sample learning to strictly avoid updates based on OOD actions \cite{iql, onestep, insample3, insample4}, while more recent approaches propose mild exploitation beyond the dataset \cite{mild2}. Alternatively, value regularization methods mitigate overestimation by penalizing critic values \cite{cql, brac, counterfactualconservativeqlearning, adversariallytrainedactorcritic} or estimating uncertainty via ensembles \cite{uncertaintybasedofflinereinforcementlearning,pessimisticbootstrappinguncertaintydrivenoffline}. Similarly, policy constraint methods explicitly restrict the learned policy to the behavior distribution \cite{td3bc, bcq,bear,supportedtrustregionoptimization}. However, relying on strict penalization or constraints often necessitates a trade-off: theoretical stability is achieved at the cost of excessive conservatism. Consequently, these methods tend to struggle in propagating signals into high-value regions that are reachable but sparsely represented. This "density bias" can lead to convergence on suboptimal modes when the optimal trajectory lacks sufficient data density.

\textbf{Generative Models in Offline RL.} Diffusion \cite{ddpm, diffusion_former} and flow matching \cite{flowmatching} excel at modeling multimodal policies \cite{visuomotorpolicy, reinflow}, but incorporating value maximization is challenging. Direct Q-gradient backpropagation \cite{diffusion_ql, flowqlearning, Ada_2024} is often computationally expensive and unstable, while simple advantage weighting \cite{ding2024diffusionbasedreinforcementlearningqweighted, kang2023efficientdiffusionpoliciesoffline} yields ineffective results. Furthermore, guidance-based methods \cite{qgpo, qipo, idql, sfbc} typically rely on complex two-stage pipelines involving separate behavior priors. To avoid these issues, we adopt Classifier-Free Guided RL (CFGRL) \cite{cfgrl}. This framework enables a single-stage training process where the policy is conditioned on optimality tokens, allowing for dynamic control between behavior cloning and value maximization at inference time.

\section{Preliminaries}

\subsection{Offline Reinforcement Learning}
The RL environment is a Markov Decision Process (MDP) defined by the tuple $(\mathcal{S}, \mathcal{A}, p_0, p, r, \gamma)$, where $\mathcal{S}$ is the state space, $\mathcal{A}$ the action space, $p_0$ the initial state distribution, $p(s'|s,a)$ the transition dynamics, $r(s,a)$ the reward function, and $\gamma \in [0,1)$ the discount factor. The agent's goal is to learn a policy $\pi(a|s)$ that maximizes the expected discounted return $\mathbb{E}_\pi\big[\sum_{t=0}^\infty \gamma^t r(s_t,a_t)\big]$.  
In the offline setting, instead of learning from interacting with the environment, the agent learns solely from a fixed dataset $\mathcal{D} = \{(s,a,r,s')\}$ collected by some behavior policy $\pi_b$.

\subsection{Implicit Q-Learning (IQL)}
Implicit Q-Learning (IQL) \cite{iql} strictly limits value estimation to the dataset support to avoid extrapolation error. It approximates the optimal value function by treating $V_\psi$ as an expectile of the Q-values via the asymmetric loss $L_2^\tau(u) = |\tau - \mathbb{I}(u < 0)|u^2$. Given the learned value, the policy is extracted using Advantage-Weighted Regression (AWR) \cite{awr1, awr2, awr}, which essentially performs behavior cloning weighted by $\exp(\alpha(Q - V))$.

\subsection{Classifier-Free Guided RL (CFGRL)}
CFGRL \cite{cfgrl} adapts classifier-free guidance to reinforcement learning to steer policies toward high-return regions without retraining. It trains a conditional flow model $v_{\phi}(a^t, t, s, o)$ where optimality is treated as a conditioning token $o \in \{\varnothing, 1\}$. At inference, the policy is refined via linear interpolation in velocity space:
\begin{equation}
    v_{\text{guided}} = (1 - w) v_{\phi}(a^t, t, s, \varnothing) + w v_{\phi}(a^t, t, s, o = 1).
\end{equation}
The guidance scale $w$ controls the trade-off between the base policy distribution and the high-value optimality condition, enabling test-time flexibility.

\section{Method}
We propose \textbf{ISEP}, a framework designed to overcome the limitations of static dataset support in offline RL. The core principle of ISEP is that optimal actions often lie in the interstices of the data support—regions that are sparsely covered or implicit in the dataset but are reachable via generalization~\cite{d4rl}. To address this, ISEP implicitly expands the action support by utilizing an interpolated value estimation while maintaining safety by anchoring to the dataset support.

To realize this, we introduce three coupled innovations: (1) An \textbf{Interpolated Value Objective} that achieves implicit action support expansion while balancing safety via a controllable parameter $p$; (2) A \textbf{Stochastic Action Selection} that prevents mode collapse when navigating non-convex and multimodal landscapes; and (3) A \textbf{Flow-Matching Instantiation} required to capture the complex, multimodal distributions aligning with our stochastic action selection strategy.

\subsection{Interpolated Value Estimation: Implicit Support Expansion}
Standard in-sample methods limit value estimation strictly to the static support of the dataset actions. This approach is overly conservative; when optimal actions are statistically sparse, standard regression is overwhelmed by the density of suboptimal data. To address this, we propose an interpolated objective that anchors values to the dataset while actively expanding support into high-value regions:

\begin{equation}
    \mathcal{L}_{V}(\psi) = (1-p) \mathbb{E}_{(s,a) \sim \mathcal{D}} \Big[ L_2^\tau\big(Q_{\hat{\theta}}(s, a) - V_\psi(s)\big) \Big] \quad + p \mathbb{E}_{\substack{s \sim \mathcal{D} \\ \hat{a} \sim \pi\phi(\cdot|s)}} \Big[ \big(Q_{\hat{\theta}}(s, \hat{a}) - V_\psi(s)\big)^2 \Big]
\label{eq:value_loss} 
\end{equation}

where $L_2^\tau$ is the asymmetric expectile loss and $\hat{a}$ is sampled from the current policy $\pi_\phi$. The parameter $p$ governs the balance between safety and exploration. While the first term ensures that value estimates remain grounded in real transitions, the second term drives \textit{Implicit Support Expansion}. By querying the Q-function with policy-generated actions $\hat{a} \sim \pi_\phi$, the method explores high-value regions that may be absent or sparse in the static dataset. These policy samples act as synthetic regression targets, effectively "densifying" the the optimal samples and allowing $V_\psi$ to capture the value of actions that are realizable but under-represented in $\mathcal{D}$. Crucially, despite exploring OOD actions, this expansion mechanism is theoretically safe. As demonstrated in Section \ref{sec:bounded_value}, $V_\psi$ remains upper-bounded by the optimal value function $V^*$.

The Q-function is updated using a standard Bellman error on the static dataset:
{\footnotesize
\begin{equation}
L_Q (\theta) = \mathbb{E}_{(s,a,s') \sim D} \left[ \left( r(s, a) + \gamma V_{\psi}(s') - Q_{\theta}(s, a) \right)^2 \right]
\label{eq:q_loss}
\end{equation}
}

\subsection{Stochastic Action Selection}\label{sec:stochastic_policy}
With the value function established, the policy must balance two potentially conflicting signals: mimicking dataset actions (behavior cloning) and moving toward high-value regions discovered by $V_\psi$ (support expansion). A standard deterministic interpolation of these gradients is problematic. As derived in \textbf{Appendix~\ref{app:mode_collapse_proof}}, minimizing a deterministic weighted sum implicitly forces the policy to target the weighted arithmetic mean of the dataset action $a_{\mathcal{D}}$ and the policy proposal $a_{\pi}$:

\begin{equation}
\tilde{\mu}_\phi(s) \approx \frac{ (1-p) \, \omega(s, a_{\mathcal{D}}) \, a_{\mathcal{D}} + p \, \omega(s, a_{\pi}) \, a_{\pi} }{ (1-p) \, \omega(s, a_{\mathcal{D}}) + p \, \omega(s, a_{\pi}) },
\end{equation}

In offline RL, the set of valid actions is often multimodal~\cite{visuomotorpolicy,florence2021implicitbehavioralcloning} and the action-value landscape is typically non-convex~\cite{li2022softmaxpolicygradientmethods}. The linear combination of two valid modes (e.g., passing an obstacle on the left versus right) frequently lies in a low-value region (collision). The naive interpolation is thus catastrophic. To prevent this mode collapse, ISEP employs a \textbf{Stochastic Action Selection}. Instead of averaging gradients, we stochastically select the target distribution at each update step using a Bernoulli gate $B \sim \text{Bernoulli}(p)$:

\begin{equation}
        \mathcal{L}_{\pi}(\phi) = (1-B) \cdot \mathbb{E}_{(s,a) \sim \mathcal{D}}[\omega(s,a)\log \pi_\phi(a \mid s)]
+B \cdot \mathbb{E}_{\substack{s \sim \mathcal{D},\\ \hat{a} \sim \pi_{\phi}(\cdot|s)}}[\omega(s,\hat{a})\log \pi_\phi(\hat{a} \mid s)],
\label{eq:policy_loss}
\end{equation}

where $\omega(s,a) = \exp(\beta( Q_{\hat{\theta}}(s, a) - V_\psi(s)))$ are advantage-based weights\cite{awr}. This formulation ensures that the policy is pushed toward a distinct, valid mode—either the conservative data sample or the optimistic policy proposal—strictly avoiding the low-value valleys inherent to deterministic interpolation.

The complete method is summarized in Algorithm~\ref{alg:algorithm1}.

\begin{figure}[ht]
\centering
\includegraphics[width=0.4\textwidth]{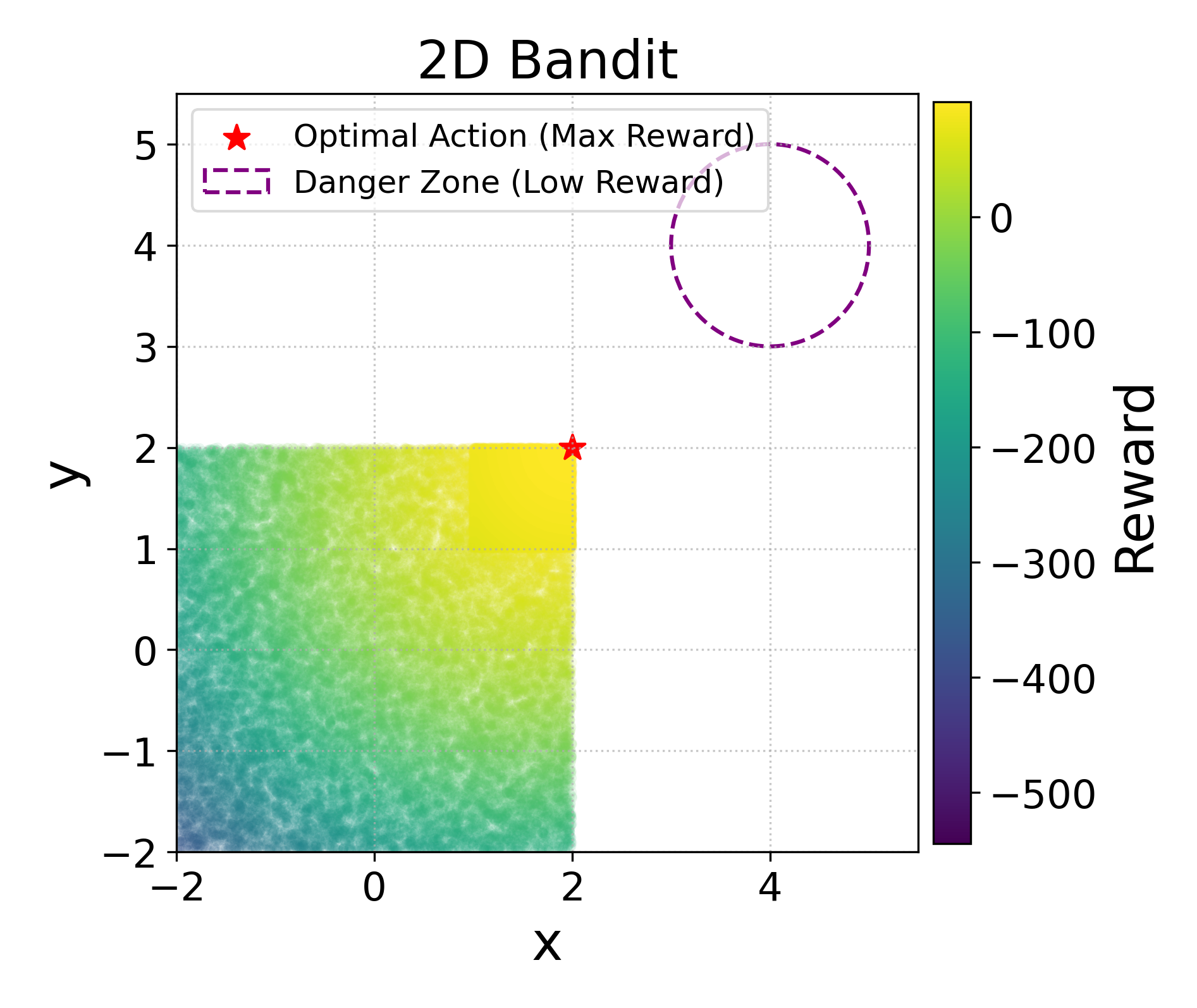}
\caption{
Visualization of the offline dataset used in the 2D bandit task.  Contour is used to indicate the reward value associated with each scatter point. The offline data consists of uniformly sampled actions from the square region $[-2, 2]^2$, shown as scattered points. The optimal action is located at $(2, 2)$ (marked by a red star). The danger zone with extremely low reward is centered at $(4, 4)$, shown as a purple circle. The task requires learning a policy that extrapolates toward the optimal region while avoiding the danger zone.
}
\label{fig:bandit_data}
\end{figure}

\begin{figure*}[t]
\centering
\includegraphics[width=1.0\textwidth]{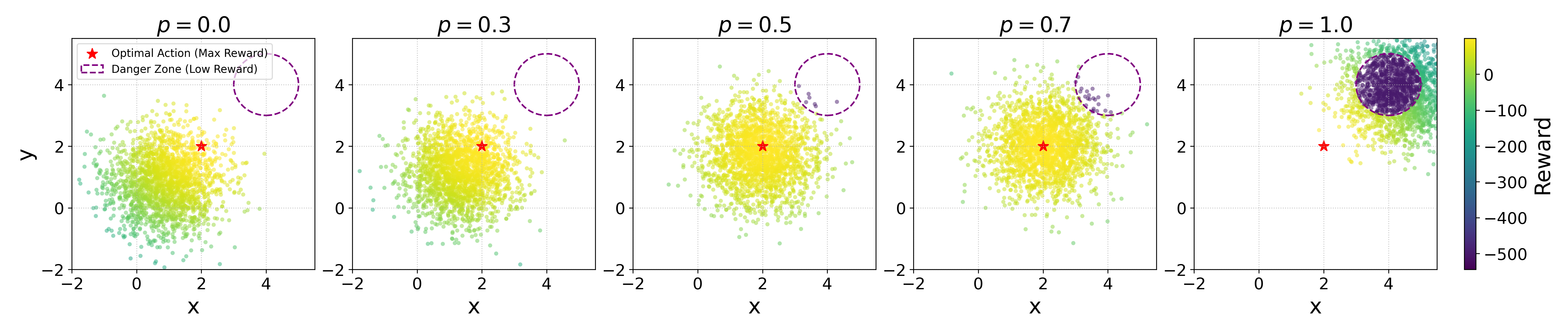}
\caption{Effect of Interpolation Parameter $p$ on Policy Performance in a 2D Bandit Task. The figure illustrates the action distributions generated by the policy for various values of $p$ (ranging from $0.0$ to $1.0$). The color contours represent the reward values, with darker regions indicating lower rewards. The red star marks the optimal center at $(2, 2)$, while the dashed purple circle delineates the danger zone at $(4, 4)$, where the reward sharply declines to $-1000$. Moderate values of $p$ (such as $0.3$ and $0.5$) allow effective generalization beyond the dataset, while aggressive values ($1.0$) lead the policy into the danger zone, demonstrating the critical role of the interpolation parameter in balancing generalization and safety.}
\label{fig:bandit_bernoulli}
\end{figure*}

\subsection{Toy Experiment: Effect of the Interpolation Parameter $p$}
We consider a continuous 2D bandit task. The offline dataset $\mathcal{D}$ contains $10{,}000$ uniform random action samples from $[-2,2]^2$ (first panel, Figure~\ref{fig:bandit_data}). The reward $R(x,y)$ defines a high-reward region centered at $(2, 2)$ with a reward of $-10\big((x - 2)^2 + (y - 2)^2\big) + 100$ and a low-reward danger zone at $(4, 4)$ with a reward of $-1000$. The objective is to learn a policy that identifies the optimal action near $(2, 2)$ while avoiding the danger zone.

At values of $p=0.3$ and $p=0.5$, the policy achieves a favorable balance, enabling effective action support expansion while maintaining safety and reducing the likelihood of entering the danger zone. This experiment illustrates how the parameter $p$ serves as a crucial control mechanism, allowing us to manage the extent of extrapolation, thereby facilitating the action support expansion without compromising safety.

\begin{algorithm}[tb]
\caption{ISEP}
\label{alg:algorithm1}
\begin{algorithmic}[1] % Changed from [2] to [1] for standard line numbering
\State Initialize policy $\pi_\phi$, Q-functions $Q_\theta$, target Q-function $Q_{\hat{\theta}}$, value function $V_\psi$
\For{each iteration}
    \State Sample transition mini-batch $\mathcal{B} = \{(s,a,r,s')\} \sim D$ 
    \State Update $V_\psi$ by minimizing Equation~\ref{eq:value_loss}
    \State Update $Q_\theta$ by minimizing Equation~\ref{eq:q_loss}
    \State Update policy $\pi_\phi$ by minimizing Equation~\ref{eq:policy_loss}
    \State Update target network: $\hat{\theta} \leftarrow \rho \hat{\theta} + (1-\rho) \theta$ 
\EndFor
\end{algorithmic}
\end{algorithm}

\begin{figure*}[t]
\centering
\includegraphics[width=1.0\textwidth]{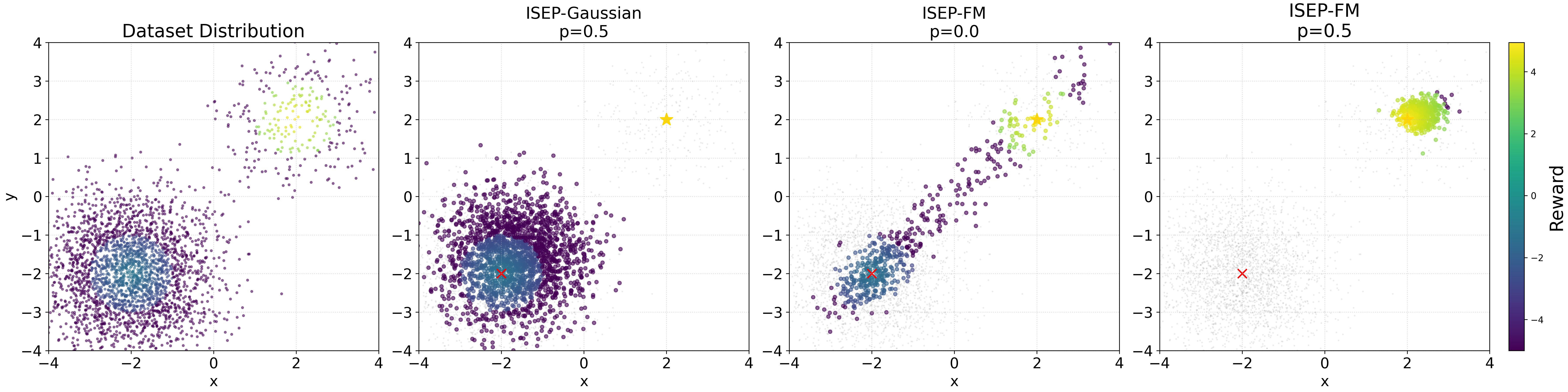}
\caption{Implicit Support Expansion on a Sparse Multimodal Bandit.
The background color contours represent the reward landscape, where darker red indicates higher reward and blue indicates the low-reward background ($r=-5$). The Suboptimal Trap center is marked with a cross ($\times$, -2, -2) and the Optimal center with a star ($\star$, 2, 2).
Dataset: Displays the training distribution, heavily skewed with $90\%$ of samples in the suboptimal trap and only $10\%$ in the optimal region.
ISEP-Gaussian: The unimodal policy averages the disconnected modes, failing to fit the data and placing probability mass in the low-reward background.
ISEP ($p=0$): Without support expansion, the flow matching policy remains conservative; despite high rewards available elsewhere, it collapses into the dense suboptimal trap.
ISEP ($p=0.5$): Our method successfully expands the support of the optimal region. By interpolating the value function, it overcomes data sparsity and shifts the policy distribution to the global optimum.}
\label{fig:fm_toy}
\end{figure*}
\subsection{Instantiation: From Gaussian to Flow}
The decision to use a Stochastic Action Selection Objective (Section~\ref{sec:stochastic_policy}) has a profound architectural implication: the target policy distribution becomes inherently multimodal.If the dataset suggests action $a_{data}$ and the expansion suggests $a_{expand}$, the optimal policy should assign probability mass to both peaks. A standard unimodal Gaussian policy cannot represent this; it would force the distribution to center between the peaks, causing the exact mode collapse we sought to avoid.Therefore, to fully utilize the ISEP framework, we necessitate a generative model capable of representing arbitrary, complex distributions. We instantiate ISEP using Conditional Flow Matching (CFM), parameterized as a velocity field.

Following CFGRL~\cite{cfgrl}, for the flow policy to effectively learn the critic signal, we condition the flow on an \textit{optimality token} $o \in \{0, 1, \varnothing\}$, where $o=1$ indicates advantage-weighted improvement ($A(s,a) \geq 0$) and $o=\varnothing$ (applied with $10\%$ dropout) represents the unconditioned distribution.

We train this model by applying the ISEP Stochastic Mixture principle directly to the flow-matching objective. The loss stochastically switches between ensuring consistency with the dataset and self-consistency with the policy's own expansion:
\begin{equation}
    \mathcal{L}_{\pi}(\phi) = 
(1-B) \cdot \mathbb{E}_{(s,a) \sim \mathcal{D}}\bigl[\| v_{\phi}(a^t, t, s, o) - (a - a^0) \|\bigr] + B \cdot \mathbb{E}_{\substack{s \sim \mathcal{D},\\ \hat{a} \sim \pi_{\phi}(\cdot|s)}}\bigl[\| v_{\phi}(\hat{a}^t, t, s, o) - (\hat{a} - \hat{a}^0) \|\bigr],
\label{eq:cfgrl_policy_loss}
\end{equation}
where $a^t = (1-t)a^0 + ta$, $\hat{a}^t = (1-t)\hat{a}^0 + t\hat{a}$, $t \sim \mathcal{U}[0,1]$, and $a^0,\hat{a}^0 \sim \mathcal{N}(0,I)$. 

At inference, we generate actions by integrating the velocity field with classifier-free guidance. Starting from $a^0 \sim \mathcal{N}(0,I)$, we compute the guided velocity at each step:
\begin{equation}
    v = (1 - w) \, v_{\phi}(a^t, t, s, \varnothing) + w \, v_{\phi}(a^t, t, s, o = 1),
\label{eq:ISEP_flow_sample}
\end{equation}
where $w \geq 0$ is the guidance weight controlling the degree of policy improvement. 

By using Flow Matching, the policy can essentially "superpose" the dataset distribution and the expanded high-value distribution without merging them destructively. At inference time, we use classifier-free guidance on the optimality token $o$ to dynamically adjust the trade-off, effectively navigating the expanded support discovered by ISEP.

\subsection{Toy Experiment: Implicit Support Expansion}

We validate the mechanism of \textit{implicit support expansion} in a sparse optimal reward setting using a 2D multimodal bandit. The dataset is highly imbalanced, featuring a dominant suboptimal mode (90\% of samples) and a sparse optimal region (10\%), separated by a low-reward background. As visualized in Figure \ref{fig:concept_fig}, baselines struggle with this imbalance. A Gaussian policy averages across modes into the low-reward background. Although expectile regression targets high-reward samples, the in-distribution baseline ($p=0$) remains overwhelmed by density bias, leaving the policy trapped in the dominant suboptimal region. In contrast, ISEP ($p=0.5$) leverages implicit support expansion to overcome this density bias. By generating interpolated samples that extend the effective range of high-value data, ISEP successfully shifts probability mass to the sparse global maximum, thereby avoiding the gravitational pull of the suboptimal mode.

\subsection{Theoretical Analysis of ISEP}\label{sec:bounded_value}
We provide a theoretical analysis of ISEP's value estimation under suboptimal datasets. Let $Q^*(s,a)$ be the optimal state-action value function for all $s\in S$ and $a \in A$. For all states $s$, $V^*(s)=\max_{a\in A} Q^*(s,a)$ The first term provides in-distribution support by leveraging expectile regression. While theoretically $\lim_{\tau \to 1}$ recovers the maximum support value \citep{iql}, in practice, limited data necessitates using $\tau \in [0.7, 0.9]$ (when $\tau \to 1$, very little data, not practical for training), making
\begin{align}
    E^\tau_{(s,a)\sim D}\big[Q^*(s, a)\big] &< \lim_{\tau' \to 1}E^{\tau'}_{(s,a)\sim D}\big[Q^*(s, a)\big]\\
    &=\max_{\substack{a \in A \\ \text{s.t. } \pi_{\beta}(a|s) > 0}}Q^*(s,a),
\end{align}
where $\tau'$ is the theoretical expectile that approaches $1$. This introduces a \textit{practical approximation gap}: the learned expectile value is strictly more conservative than the true maximum support value.
\begin{assumption}[Expectile Gap]\label{ass:expectile-gap}
Due to finite sample sizes so $\tau < 1-\epsilon$, where $\epsilon \in (0, 1)$ is non trivial, the learned expectile value $V_\psi^\tau(s)$ underestimates the dataset maximum, for any states $s$ in dataset:
\begin{align}
  \mathbb{E}^{\tau}_{a\sim D(s)}[Q^*(s,a)] \leq \max_{a \in \mathcal{D}(s)} Q^*(s,a) - \delta_{\tau}
\end{align}
\end{assumption}
Besides, there is also this data sub-optimality assumption:
\begin{assumption}[Dataset Sub-optimality]\label{ass:dataset-suboptimal}
Let \( Q^*(s,a) \) denote the optimal state-action value function under the true underlying MDP.  There exist non trivial constants $\delta_{sub} > 0$, such that for any state $s$:
\begin{align}
    \max_{a\in D(s)}Q^*(s,a) &\leq \max_{a' \in \mathcal{A}} Q^*(s,a') - \delta_{sub}\\
    &\leq V^*(s) - \delta_{sub}
\end{align}
\end{assumption}
This assumption captures realistic settings where datasets contain mostly suboptimal trajectories but may include near-optimal actions with some probability $\eta$. The parameter $\delta_{sub}$ quantifies the degree of suboptimality, while $\eta$ represents the coverage of high-quality actions.

\begin{theorem}[ISEP Safety Guarantee]\label{thm:ISEP-safety}
Let $\hat{V}(s)$ be the value function minimizing Equation~\ref{eq:value_loss}.  Under Assumption~\ref{ass:expectile-gap} and Assumption~\ref{ass:dataset-suboptimal}, if for all state $s\in S$, the interpolation parameter $p$ satisfies:
\begin{align}
    p \leq \frac{\delta_\tau + \delta_{sub}}{V_{\max}- V^*(s) + \delta_\tau + \delta_{sub}}
\label{eq:p_bound}
\end{align}
then, with probability at least $\eta$ (governed by the dataset quality in Assumption \ref{ass:dataset-suboptimal}), we have:
\begin{align}
    \hat{V}(s) \leq V^*(s),
\end{align}
 where $R_{\max}$ is the maximum reward and $V_{\max}=\frac{2R_{\max}}{1-\gamma}$.
\end{theorem}

\begin{proof}
See Appendix~\ref{app:maxium_bound}.
\end{proof}

Theorem~\ref{thm:ISEP-safety} provides a principled guideline for selecting the parameter \( p \) based on dataset characteristics. The derived bound reveals that a more aggressive strategy—selecting a larger \( p \)—is permissible when the dataset exhibits high suboptimality (i.e., a larger value of \( \delta_{sub} \)). Conversely, for datasets of higher quality, a more conservative approach is required, necessitating a value of \( p \) closer to $0$ to prioritize in-distribution support.

This theoretical result highlights ISEP's adaptive nature: by adjusting $p$, we can control the trade-off between conservative value estimation and policy improvement based on dataset quality. The bound ensures that even when incorporating policy-generated actions through the $p$ term, ISEP's value estimates remain bounded by optimal state value.

%%%%%%%%%%%%%%%%%%%%
%%%%%%%%%%%%%%%%%%%%%%%%%%%%%%%%%%%%

\begin{table*}[t]
\centering

\setlength{\tabcolsep}{4.5pt}
\caption{Normalized scores on MuJoCo locomotion tasks. Best overall scores are shown in \textbf{bold}; best scores among methods not using diffusion/flow matching are \underline{underlined}. Averaged over 5 seeds. Baseline results are from original publications.}
\begin{tabular}{l|c|c|c|c|c|c|c||c|c}
\hline
\textbf{Environment} & \textbf{BC} & \textbf{CQL} & \textbf{IQL} & \textbf{DQL} & \textbf{IDQL} & \textbf{QGPO} & \textbf{QIPO} & \textbf{ISEP} & \textbf{ISEP-FM}\\
\hline
halfcheetah-med & 42.6 & 44.4 & 47.5 & 51.1 & 49.7 & 54.1 & 54.1 & \underline{49.3 $\pm$ 0.2} & \textbf{59.9 $\pm$ 0.3}\\
hopper-med & 52.9 & 79.2 & 67.0 & 90.5 & 63.1 & 98.0 & 94.0 & \underline{79.1$\pm$3.5} & \textbf{101.3 $\pm$ 2.1} \\
walker2d-med & 75.3 & 58.0 & 78.3 & 87.0 & 80.2 & 86.0 & 87.6 & \underline{82.1$\pm$0.4} & \textbf{90.4 $\pm$ 1.1}\\
\hline
halfcheetah-med-rep & 36.6 & 45.5 & 44.5 & 47.8 & 45.1 & 47.6 & 48.0 & \underline{46.3$\pm$0.1} & \textbf{54.6 $\pm$ 0.5}\\
hopper-med-rep & 18.1 & 95.0 & 94.7 & 101.3 & 82.4 & 96.9 & 101.2 & \underline{102.7$\pm$0.6} & \textbf{104.5 $\pm$ 0.8}\\
walker2d-med-rep & 26.0 & 77.2 & 73.9 & \textbf{95.5} & 79.8 & 84.4 & 90.1 & \underline{82.0$\pm$3.6} & 92.7 $\pm$ 1.5\\
\hline
halfcheetach-med-exp & 55.2 & 62.4 & 86.7 & 94.4 & 90.3 & 93.5 & 94.4 & \underline{92.4$\pm$0.3} &  \textbf{95.7 $\pm$ 0.2}\\
hopper-med-exp & 52.5 & 111.0 & 91.5 & 105.3 & 111.9 & 108.0 & 112.1 & \underline{112.0$\pm$0.2} & \textbf{112.3 $\pm$ 0.3}\\
walker2d-med-exp & 107.5 & 98.7 & 109.6 & 111.6 & 111.2 & 110.7 & 110.8 & \underline{110.2$\pm$0.5} & \textbf{112.5 $\pm$ 0.4}\\
\hline
\textbf{locomotion total} & 51.9 & 74.6 & 77.0 & 88.0  & 79.1 & 86.6 & 88.1 & \underline{84.0} & \textbf{91.5}\\
\hline
\end{tabular}
\label{tab:1}
\end{table*}

\begin{table*}[t]
\centering
\caption{Normalized scores on MuJoCo locomotion tasks. Best overall scores are shown in \textbf{bold}; best scores among methods not using diffusion/flow matching are \underline{underlined}. Averaged over 5 seeds. Results of baseline methods are taken from their original papers, except IQL~\cite{iql} results obtained by running official code.}
\begin{tabular}{l|c|c|c|c||c|c}
\hline
\textbf{Environment} & \textbf{BC} & \textbf{CQL} & \textbf{IQL} & \textbf{DQL} & \textbf{ISEP} & \textbf{ISEP-FM}  \\
\hline
pen-human & 25.8 & 74.0 & 71.5 & 72.8  & 80.9 $\pm$ 1.1   & \textbf{82.7 $\pm$ 0.7} \\
pen-cloned & 38.3 & 40.3 & 79.5 & 57.3 & 84.5 $\pm$ 2.3  & \textbf{85.8 $\pm$ 1.6}\\
\hline
\textbf{adroit total} & 32.1 & 57.1 & 75.5 & 65.1 & 82.7 & \textbf{84.3}\\
\hline
kitchen-complete & 33.8 & 43.8 & 65.0 & 84.0 &  70.5 $\pm$ 4.2 & \textbf{87.5 $\pm$ 3.7}\\
kitchen-partial & 33.8 & 49.8 & 46.3 & 60.5 &  47.2 $\pm$ 3.6 & \textbf{70.6 $\pm$ 2.2}\\
kitchen-mixed & 47.5 & 51.0 & 51.0 & 62.6 &  58.5 $\pm$ 1.2 & \textbf{66.2 $\pm$ 1.5}\\
\hline
\textbf{kitchen total} & 38.4 &48.2  & 54.1 & 69.0 & \underline{58.6} & \textbf{74.8} \\
\hline
\end{tabular}
\label{tab:2}
\end{table*}

\section{Experiments}
We evaluate our method by comparing it with prior offline RL approaches on the D4RL benchmark~\cite{d4rl}. Additionally, we conduct empirical studies to investigate how the Bernoulli parameter influences the trade-off between safety and extrapolation in offline RL. Finally, we perform an ablation study on the two main components of our algorithm to assess their individual contributions to overall performance.

\subsection{Comparison on Offline RL Benchmarks}
We evaluate the performance of our methods against prior offline RL methods using the D4RL benchmark suite~\cite{d4rl}, with results presented in Table~\ref{tab:1} and Table~\ref{tab:2}. The benchmark includes tasks from three domains: Gym-MuJoCo locomotion, Adroit robotic hand manipulation, and Kitchen sequential object manipulation.

\subsubsection{Datasets}
We evaluate our methods against prior offline RL algorithms using the D4RL benchmark suite~\cite{d4rl}, with results presented in Table~\ref{tab:1} and Table~\ref{tab:2}. The benchmark encompasses three challenging domains: Gym-MuJoCo locomotion tasks with continuous control in relatively well-covered state-action spaces; Adroit robotic hand manipulation with high-dimensional control from human demonstrations requiring strong policy regularization to remain within narrow data support; and Kitchen sequential object manipulation emphasizing long-horizon planning and temporal credit assignment from sparse, human-collected demonstrations.

\subsubsection{Implementation Details}
We employ a multi-layer perceptron (MLP) for the Gaussian policy and an MLP with Mish activation \cite{mish} as the backbone for the flow-based policy. Training is conducted for 2000 epochs on Gym tasks and 1000 epochs on Adroit and Kitchen. Each epoch comprises 1000 training steps with a batch size of 256. Optimization is performed using the Adam optimizer \cite{adam}. For evaluation, we generate 20 rollout trajectories for Gym tasks and 50 trajectories for both Adroit and Kitchen tasks. Final results are reported as the average performance over five random seeds. A comprehensive list of task-specific hyperparameters, including the interpolation parameter $p$, expectile values $\tau$, and temperature $\beta$ or guidance weight ($w$), are provided in Appendix \ref{sec:appendix_hyperparams}.

\subsubsection{Baseline} Our evaluation compares our methods with several representative offline RL baselines, including CQL~\cite{cql}, IQL~\cite{iql}, Diffusion-QL (DQL)~\cite{diffusion_ql}, IDQL~\cite{idql}, QGPO~\cite{qgpo}, QIPO~\cite{qipo} and behavioral cloning (BC). Results of baseline methods are taken either from their original papers or running their official code.

The comparative results for ISEP and ISEP-FM against these baselines are presented in Table~\ref{tab:1} for Gym-MuJoCo locomotion tasks and Table~\ref{tab:2} for the Adroit and Kitchen domains.

\subsection{Runtime}
All experiments were run on an NVIDIA RTX 4090 GPU. Our Gaussian variant (ISEP) trains in  1.5-2 hours (1.5M steps) to converge, matching the efficiency of IQL~\cite{iql}. The more expressive flow variant (ISEP‑FM) requires approximately 4–5 hours (2M steps) to converge, which is slower but remains faster than full diffusion-based methods.

\subsection{Ablation Study}
\begin{figure}[htbp]
    \centering
    \begin{subfigure}[t]{0.24\linewidth}
        \centering
        \includegraphics[width=\linewidth]{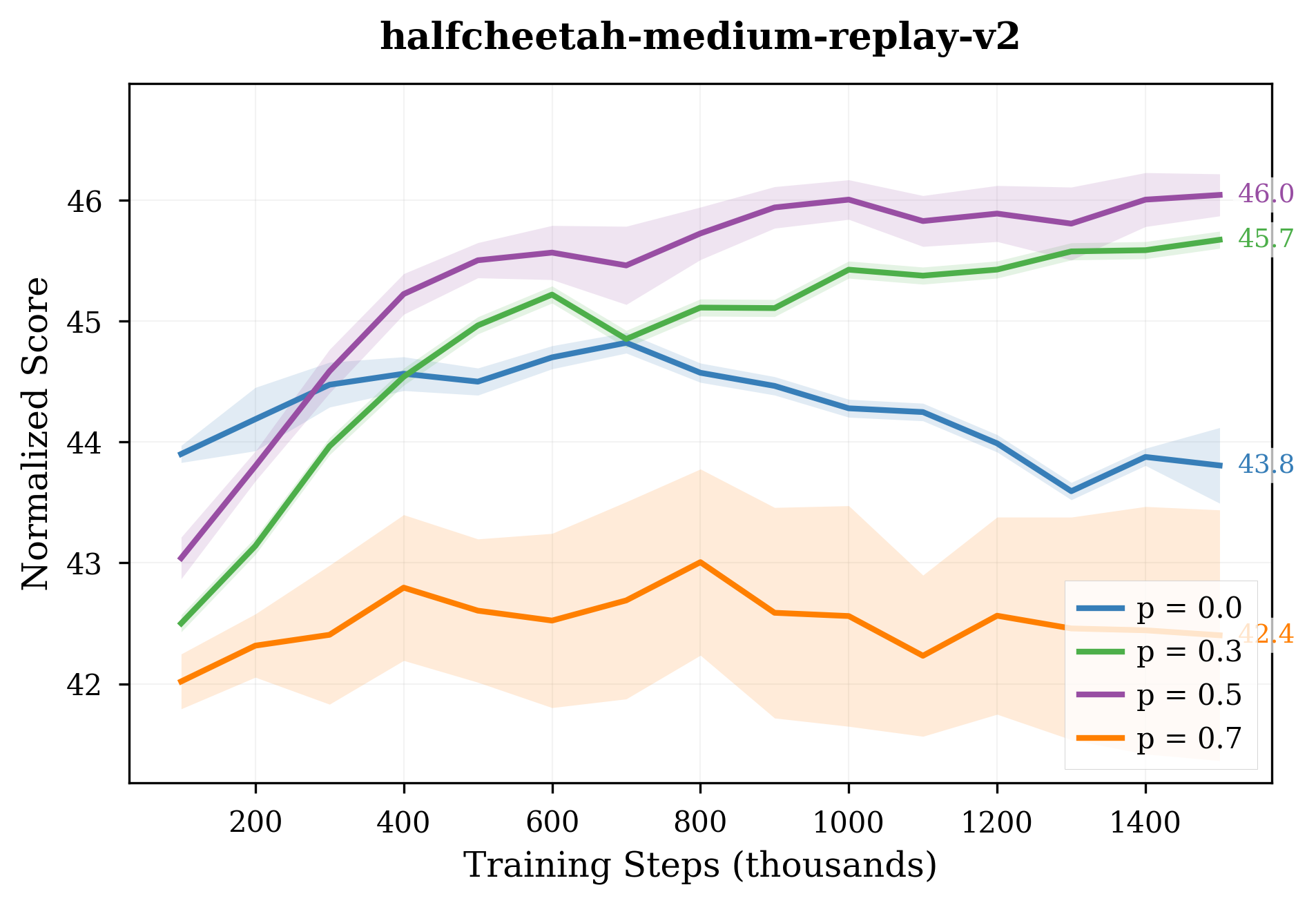}
        \label{fig:p_halfcheetah_replay}
    \end{subfigure}
    \hfill
    \begin{subfigure}[t]{0.24\linewidth}
        \centering
        \includegraphics[width=\linewidth]{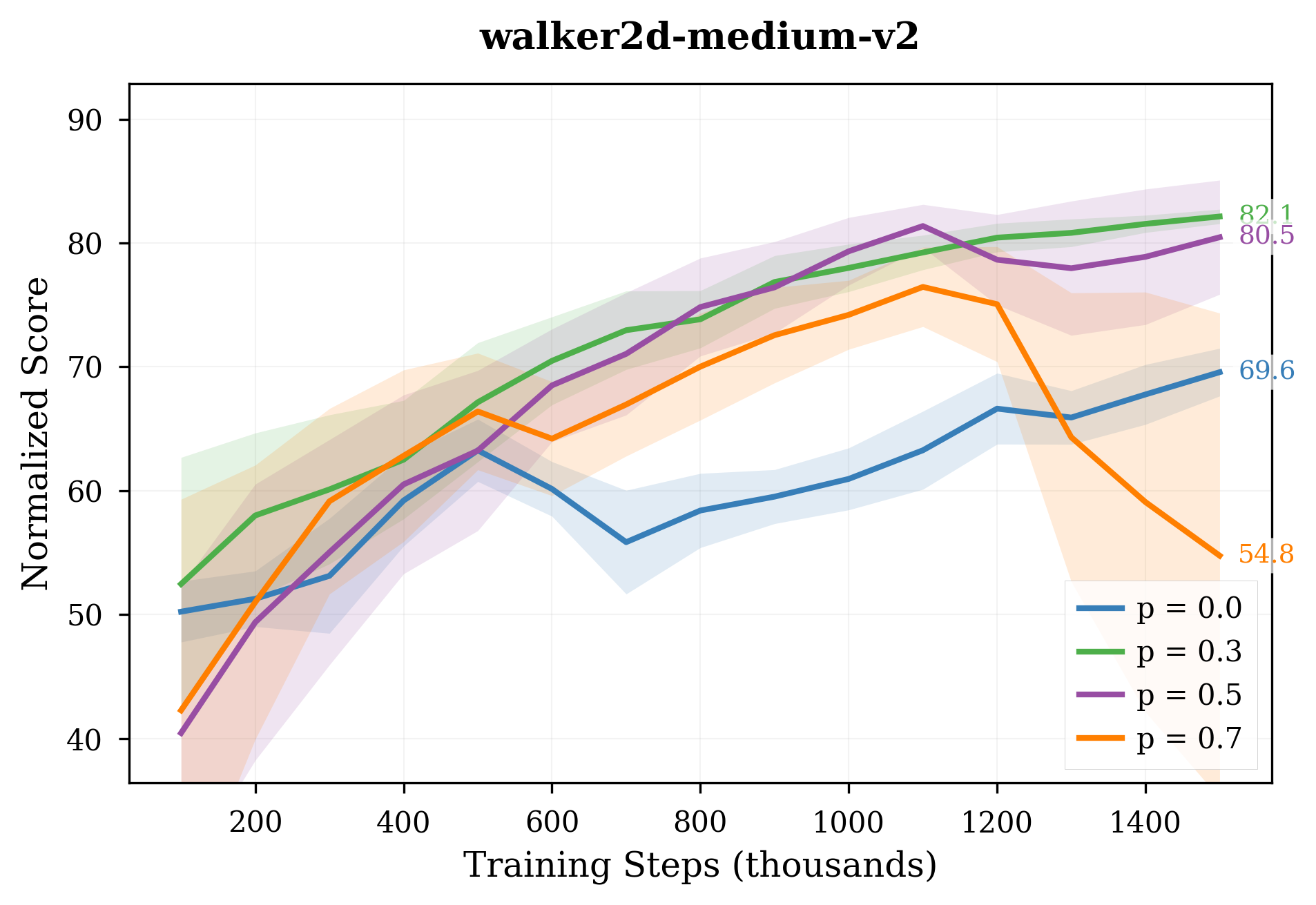}
        \label{fig:p_walker2d}
    \end{subfigure}
    \hfill
    \begin{subfigure}[t]{0.24\linewidth}
        \centering
        \includegraphics[width=\linewidth]{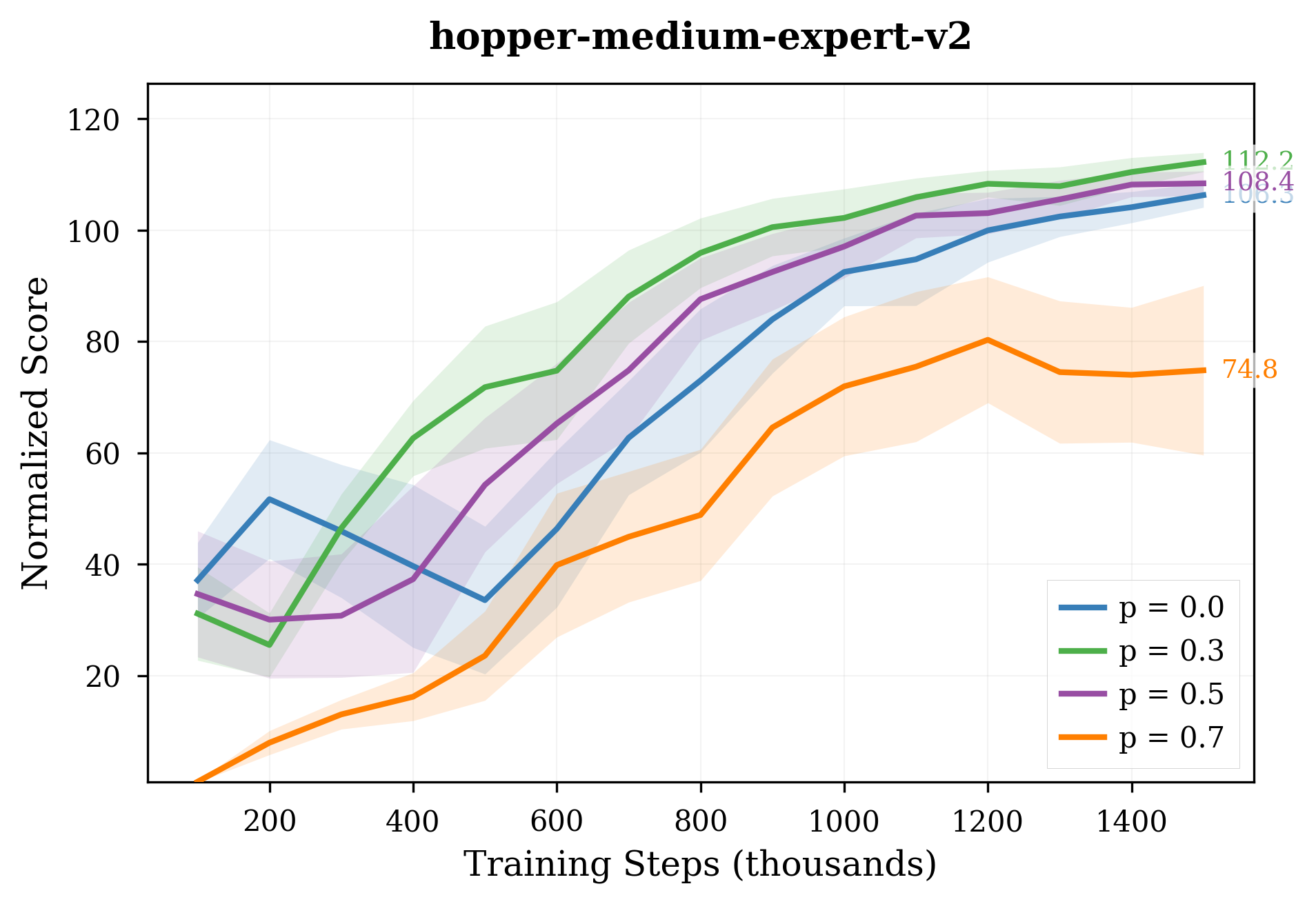}
        \label{fig:p_hopper}
    \end{subfigure}
    \hfill
    \begin{subfigure}[t]{0.24\linewidth}
        \centering
        \includegraphics[width=\linewidth]{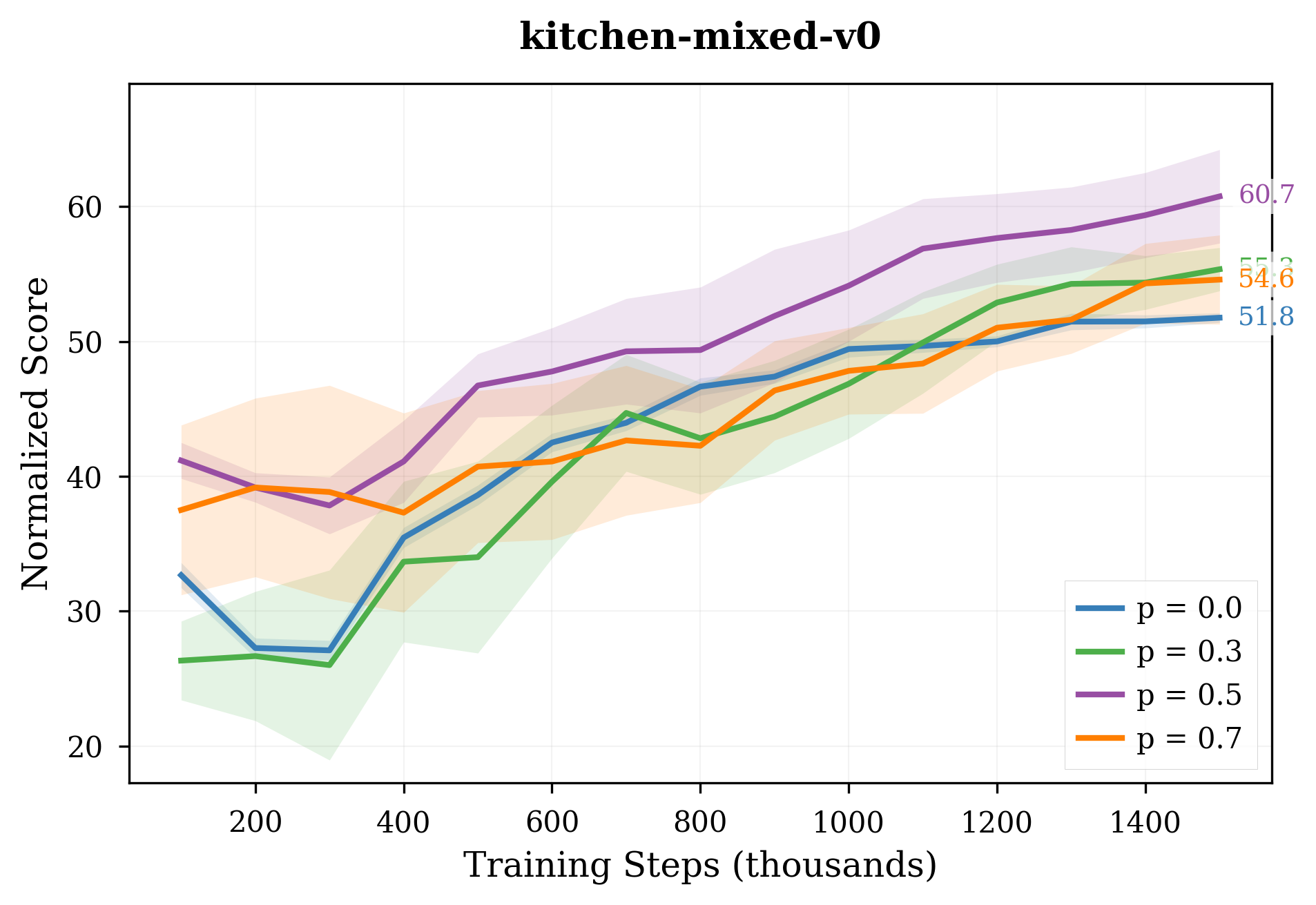}
        \label{fig:p_kitchen}
    \end{subfigure}

    \caption{\textbf{Sensitivity analysis of the interpolation parameter $p$ across representative D4RL tasks.}
    Each subplot shows the D4RL Normalized Score over training steps for different values of $p$,
    spanning locomotion and manipulation
    domains. Moderate values ($p \in \{0.3, 0.5\}$) consistently yield
    the highest normalized scores across all environments. Extreme values ($p \ge 0.9$) diverge
    and are omitted for visual clarity.
    Solid lines show the mean across seeds with EMA smoothing (weight $= 0.6$);
    shaded regions denote $\pm 1$ SEM.}
    \label{fig:bernoulli_locomotion}
\end{figure}

\begin{figure}[h]
\centering
\includegraphics[width=0.5\textwidth]{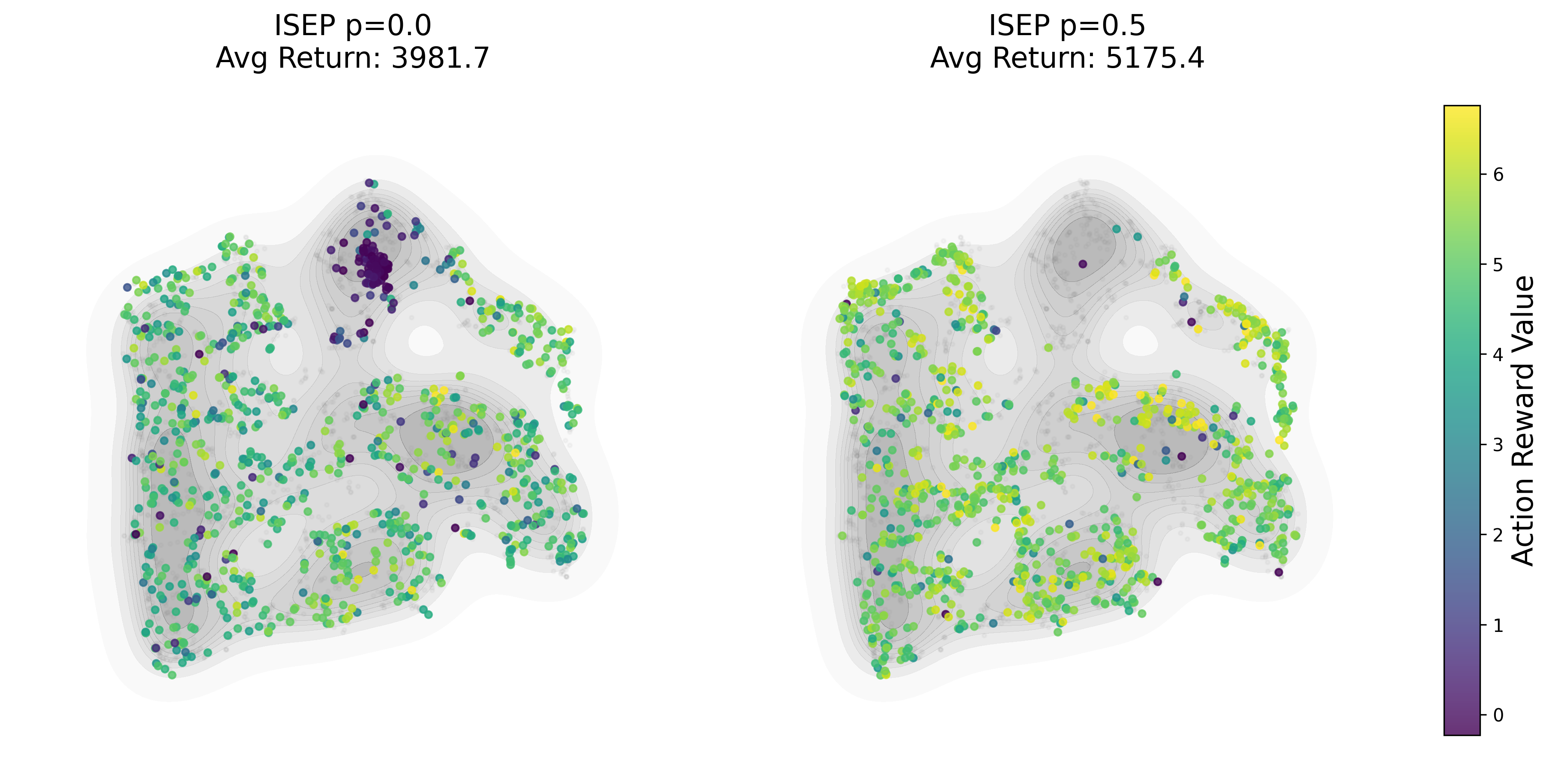}%
 \caption{\textbf{t-SNE Visualization of Action Support Expansion on \texttt{halfcheetah-medium-replay-v2}.}     We project sampled actions into 2D space, with color contours indicating reward magnitude (brighter colors denote higher rewards).     \textbf{Left:} The strictly in-distribution baseline ($p=0.0$) remains trapped in dominant low-reward regions due to dataset density bias.\textbf{Right:} ISEP ($p=0.5$) exhibits \textit{implicit action support expansion}, successfully escaping the low-reward cluster to cover sparse, high-value regions specific to this environment.}    \label{fig:halfcheetah-ood}
\end{figure}

In this section, we analyze the source of ISEP's performance gains on D4RL tasks. We dissect the two core contributions of our method: the Bernoulli-gated value interpolation mechanism and the flow-based policy extraction via CFGRL.

\subsubsection{Effect of the Interpolation Parameter $p$}
We evaluate the sensitivity of the interpolation parameter $p$ across a diverse suite of D4RL tasks, encompassing both locomotion (\subref{fig:p_halfcheetah_replay}--\subref{fig:p_hopper}) and manipulation (\subref{fig:p_kitchen}) domains. As illustrated in Figure \ref{fig:bernoulli_locomotion}, moderate values of the interpolation parameter, specifically $p \in \{0.3, 0.5\}$, consistently yield the highest normalized scores across all tested environments. These balanced settings significantly outperform the purely conservative baseline ($p=0.0$) by enabling the policy to leverage higher-value regions slightly beyond the immediate dataset support. 

However, we observe that excessive optimism is detrimental; configurations with $p \ge 0.9$ are omitted from the plots as they lead to catastrophic divergence and negative scores across all tasks. This failure stems from unconstrained extrapolation in the value function, which destabilizes training. These results reinforce the necessity of a "middle-ground" approach: while some degree of optimism is required to improve upon the behavior policy, the learning process must remain sufficiently anchored to the data distribution to maintain numerical stability.

\subsubsection{Implicit Action Support Expansion}
We analyze the ability of ISEP to perform \textit{implicit action support expansion} by visualizing the learned action space on \texttt{halfcheetah-medium-replay-v2} using t-SNE~\cite{tsne} as shown in Figure \ref{fig:halfcheetah-ood}. A critical limitation of strictly in-distribution methods ($p=0.0$, equivalent to IQL) is their susceptibility to density bias: even when using expectile regression to target high returns, the policy is inevitably "dragged" towards suboptimal modes if they dominate the dataset density. In contrast, ISEP ($p=0.5$) utilizes interpolated samples to locally expand the support of high-quality data. This effectively liberates the policy from the gravitational pull of the dense, low-reward regions found in the replay buffer, allowing it to shift probability mass towards sparse, high-value actions that baselines fail to capture.

\subsubsection{Stochastic Selection vs. Deterministic Interpolation}
\label{sec:ablation_action_mix}
Another core design choice in ISEP is the use of stochastic \textit{action switching} rather than averaging. To validate this, we compare ISEP against a \textit{Deterministic Action Interpolation} baseline on \texttt{halfcheetah-medium-replay-v2} and \texttt{walker2d-medium-v2}. In this baseline, the policy update uses interpolation of $a_D$ and $a_\pi$, rather than stochastic selection.

As shown in Figure \ref{fig:stochastic-action-ablation}, the deterministic baseline suffers a distinct performance drop compared to ISEP. We attribute this to the "off-manifold" problem inherent in naive averaging. In complex locomotion tasks, the valid action space is often non-convex; a linear combination of two valid actions (e.g., two distinct gaits) frequently results in a physically unstable or invalid posture. By employing stochastic selection, ISEP ensures that the Q-function is queried only with realizable actions from either the policy or the dataset, thereby avoiding mode collapse.

\begin{figure}[h]
\centering
\begin{subfigure}[b]{0.35\textwidth}
    \centering
    \includegraphics[width=\linewidth]{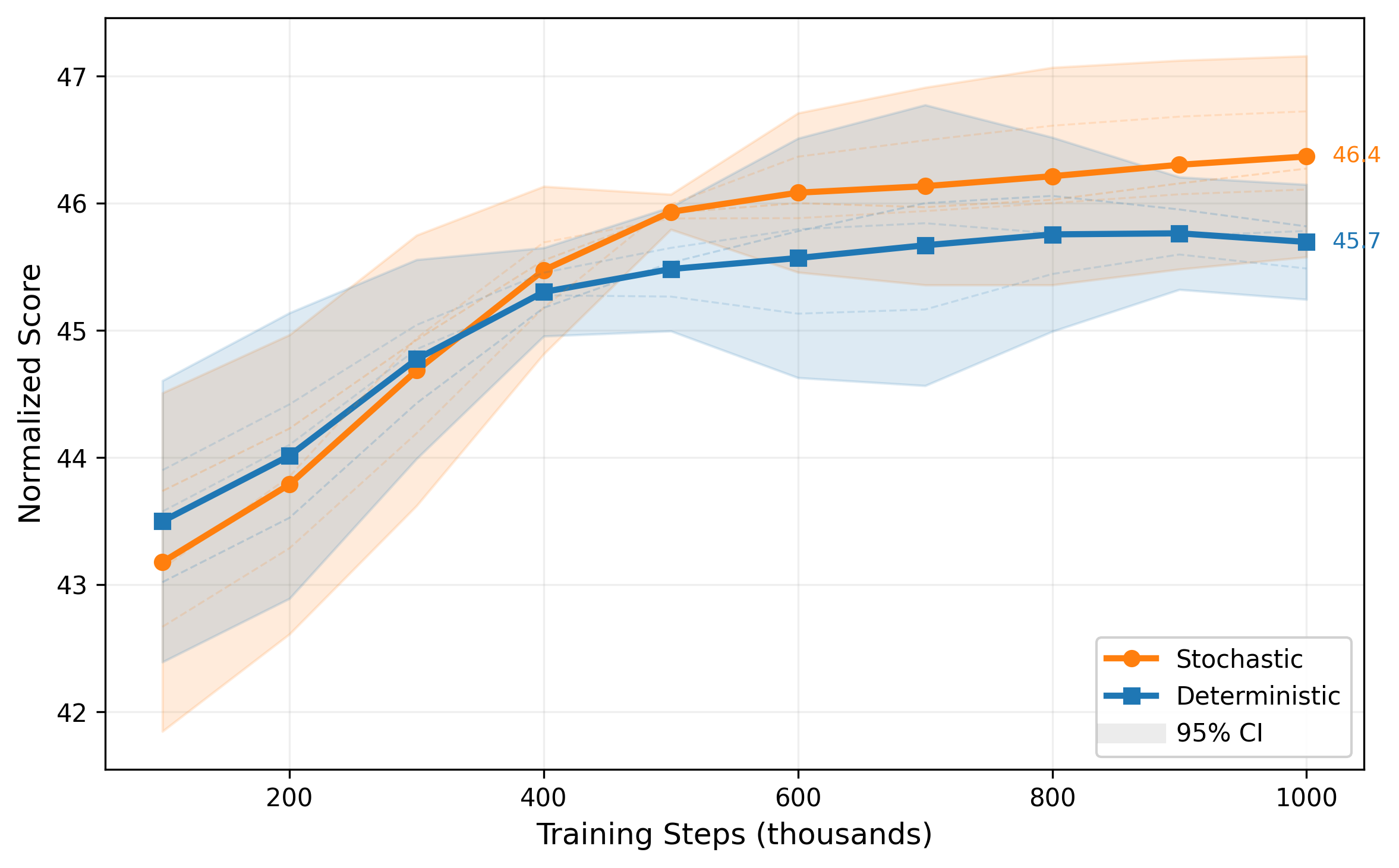}
    \caption{halfcheetah-medium-replay-v2}
    \label{fig:ablation-cheetah}
\end{subfigure}
\hspace{0.05\textwidth}
\begin{subfigure}[b]{0.35\textwidth}
    \centering
    \includegraphics[width=\linewidth]{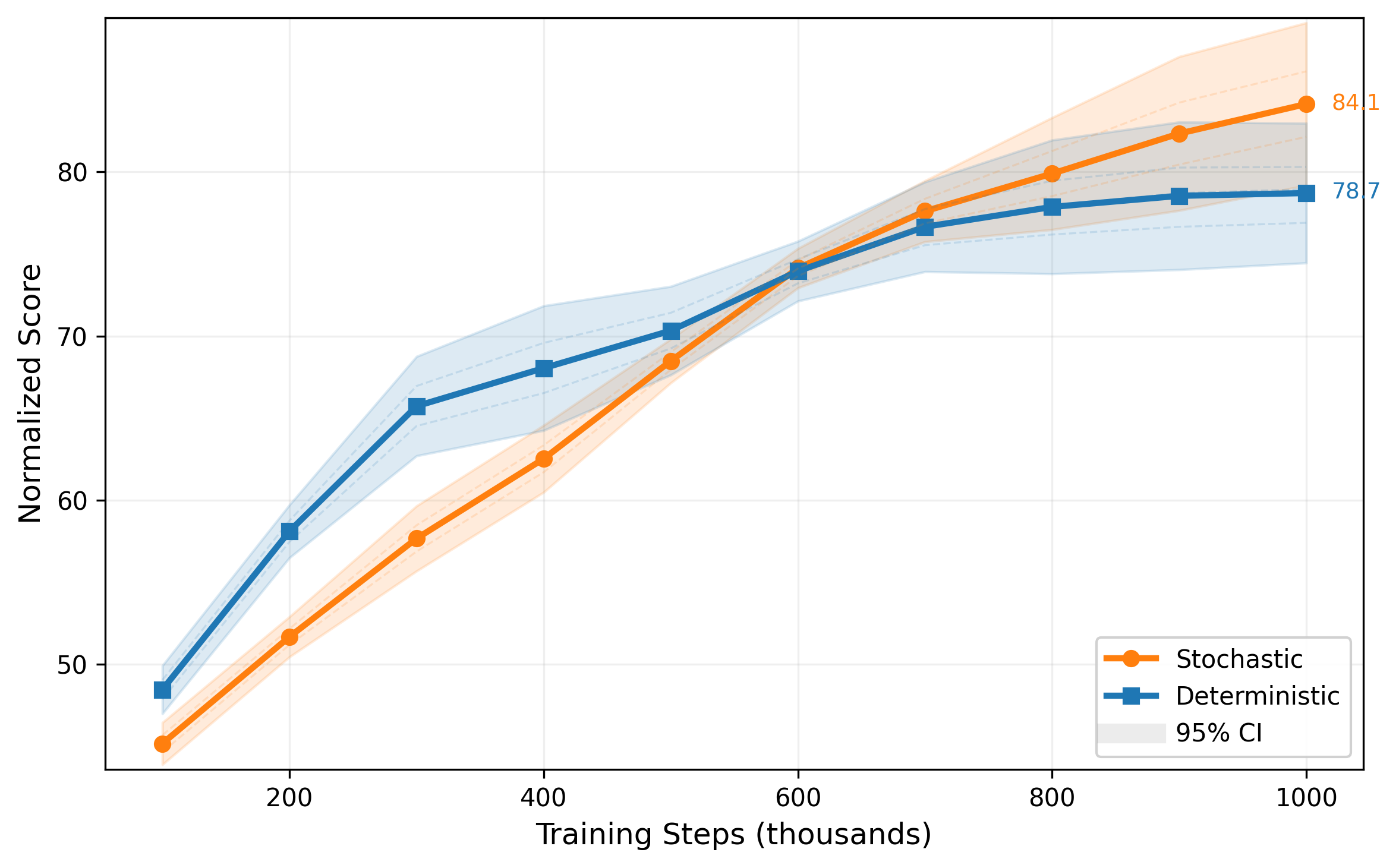}
    \caption{walker2d-medium-v2}
    \label{fig:ablation-walker}
\end{subfigure}
\caption{\textbf{Stochastic Selection vs. Deterministic Interpolation.} Ablation study comparing ISEP (labeled \textit{Stochastic}) against naive action interpolation (labeled \textit{Deterministic}) on (a) halfcheetah-medium-replay and (b) walker2d-medium. The deterministic baseline underperforms because linear averaging between distinct modes often yields "off-manifold" actions. In contrast, stochastic action selection maintains valid behavioral modes, leading to consistent performance gains.}
\label{fig:stochastic-action-ablation}
\end{figure}

\section{Conclusion}
In this work, we presented the Implicit Support Expansion via stochastic Policy optimization (ISEP), a framework designed to safely extrapolate beyond static datasets in offline RL. ISEP balances conservatism and optimism through three coupled innovations: (1) a hybrid value objective that dynamically expands support into high-reward regions; (2) a stochastic action selection strategy that prevents mode collapse in non-convex landscapes; and (3) an instantiation using Conditional Flow Matching to capture the resulting multimodal policy distributions. We theoretically demonstrated that our interpolation parameter $p$ serves as a rigorous safety mechanism, guaranteeing bounded value divergence. Our results suggest that by combining support expansion with the expressivity of modern generative models, offline RL agents can effectively discover superior actions without sacrificing safety.
%Bibliography
\bibliographystyle{unsrt}  
\bibliography{references}  

\begin{thebibliography}{10}

\bibitem{fujimoto2019offpolicydeepreinforcementlearning}
Scott Fujimoto, David Meger, and Doina Precup.
\newblock Off-policy deep reinforcement learning without exploration, 2019.

\bibitem{offlinerl_survey}
Sergey Levine, Aviral Kumar, George Tucker, and Justin Fu.
\newblock Offline reinforcement learning: Tutorial, review, and perspectives on open problems.
\newblock {\em CoRR}, abs/2005.01643, 2020.

\bibitem{offlinerl_survey2}
Rafael Figueiredo~Prudencio, Marcos R. O.~A. Maximo, and Esther~Luna Colombini.
\newblock A survey on offline reinforcement learning: Taxonomy, review, and open problems.
\newblock {\em IEEE Transactions on Neural Networks and Learning Systems}, 35(8):10237–10257, August 2024.

\bibitem{iql}
Ilya Kostrikov, Ashvin Nair, and Sergey Levine.
\newblock Offline reinforcement learning with implicit q-learning, 2021.

\bibitem{eql}
Divyansh Garg, Joey Hejna, Matthieu Geist, and Stefano Ermon.
\newblock Extreme q-learning: Maxent rl without entropy, 2023.

\bibitem{sql}
Haoran Xu, Li~Jiang, Jianxiong Li, Zhuoran Yang, Zhaoran Wang, Victor Wai~Kin Chan, and Xianyuan Zhan.
\newblock Offline rl with no ood actions: In-sample learning via implicit value regularization, 2023.

\bibitem{bcq}
Scott Fujimoto, David Meger, and Doina Precup.
\newblock Off-policy deep reinforcement learning without exploration, 2019.

\bibitem{td3bc}
Scott Fujimoto and Shixiang~Shane Gu.
\newblock A minimalist approach to offline reinforcement learning, 2021.

\bibitem{bear}
Aviral Kumar, Justin Fu, George Tucker, and Sergey Levine.
\newblock Stabilizing off-policy q-learning via bootstrapping error reduction.
\newblock {\em CoRR}, abs/1906.00949, 2019.

\bibitem{spot}
Jialong Wu, Haixu Wu, Zihan Qiu, Jianmin Wang, and Mingsheng Long.
\newblock Supported policy optimization for offline reinforcement learning, 2022.

\bibitem{brac}
Yifan Wu, George Tucker, and Ofir Nachum.
\newblock Behavior regularized offline reinforcement learning, 2019.

\bibitem{cql}
Aviral Kumar, Aurick Zhou, George Tucker, and Sergey Levine.
\newblock Conservative q-learning for offline reinforcement learning.
\newblock {\em CoRR}, abs/2006.04779, 2020.

\bibitem{cfgrl}
Kevin Frans, Seohong Park, Pieter Abbeel, and Sergey Levine.
\newblock Diffusion guidance is a controllable policy improvement operator, 2025.

\bibitem{onestep}
David Brandfonbrener, William~F. Whitney, Rajesh Ranganath, and Joan Bruna.
\newblock Offline rl without off-policy evaluation, 2021.

\bibitem{insample3}
Haoran Xu, Li~Jiang, Jianxiong Li, Zhuoran Yang, Zhaoran Wang, Victor Wai~Kin Chan, and Xianyuan Zhan.
\newblock Offline rl with no ood actions: In-sample learning via implicit value regularization, 2023.

\bibitem{insample4}
Divyansh Garg, Joey Hejna, Matthieu Geist, and Stefano Ermon.
\newblock Extreme q-learning: Maxent rl without entropy, 2023.

\bibitem{mild2}
Yixiu Mao, Qi~Wang, Yun Qu, Yuhang Jiang, and Xiangyang Ji.
\newblock Doubly mild generalization for offline reinforcement learning, 2024.

\bibitem{counterfactualconservativeqlearning}
Jianzhun Shao, Yun Qu, Chen Chen, Hongchang Zhang, and Xiangyang Ji.
\newblock Counterfactual conservative q learning for offline multi-agent reinforcement learning, 2023.

\bibitem{adversariallytrainedactorcritic}
Ching-An Cheng, Tengyang Xie, Nan Jiang, and Alekh Agarwal.
\newblock Adversarially trained actor critic for offline reinforcement learning, 2022.

\bibitem{uncertaintybasedofflinereinforcementlearning}
Gaon An, Seungyong Moon, Jang-Hyun Kim, and Hyun~Oh Song.
\newblock Uncertainty-based offline reinforcement learning with diversified q-ensemble, 2021.

\bibitem{pessimisticbootstrappinguncertaintydrivenoffline}
Chenjia Bai, Lingxiao Wang, Zhuoran Yang, Zhihong Deng, Animesh Garg, Peng Liu, and Zhaoran Wang.
\newblock Pessimistic bootstrapping for uncertainty-driven offline reinforcement learning, 2022.

\bibitem{supportedtrustregionoptimization}
Yixiu Mao, Hongchang Zhang, Chen Chen, Yi~Xu, and Xiangyang Ji.
\newblock Supported trust region optimization for offline reinforcement learning, 2023.

\bibitem{ddpm}
Jonathan Ho, Ajay Jain, and Pieter Abbeel.
\newblock Denoising diffusion probabilistic models, 2020.

\bibitem{diffusion_former}
Jascha Sohl-Dickstein, Eric~A. Weiss, Niru Maheswaranathan, and Surya Ganguli.
\newblock Deep unsupervised learning using nonequilibrium thermodynamics, 2015.

\bibitem{flowmatching}
Yaron Lipman, Ricky T.~Q. Chen, Heli Ben-Hamu, Maximilian Nickel, and Matt Le.
\newblock Flow matching for generative modeling, 2023.

\bibitem{visuomotorpolicy}
Cheng Chi, Zhenjia Xu, Siyuan Feng, Eric Cousineau, Yilun Du, Benjamin Burchfiel, Russ Tedrake, and Shuran Song.
\newblock Diffusion policy: Visuomotor policy learning via action diffusion, 2024.

\bibitem{reinflow}
Tonghe Zhang, Chao Yu, Sichang Su, and Yu~Wang.
\newblock Reinflow: Fine-tuning flow matching policy with online reinforcement learning, 2025.

\bibitem{diffusion_ql}
Zhendong Wang, Jonathan~J Hunt, and Mingyuan Zhou.
\newblock Diffusion policies as an expressive policy class for offline reinforcement learning, 2023.

\bibitem{flowqlearning}
Seohong Park, Qiyang Li, and Sergey Levine.
\newblock Flow q-learning, 2025.

\bibitem{Ada_2024}
Suzan~Ece Ada, Erhan Oztop, and Emre Ugur.
\newblock Diffusion policies for out-of-distribution generalization in offline reinforcement learning.
\newblock {\em IEEE Robotics and Automation Letters}, 9(4):3116–3123, April 2024.

\bibitem{ding2024diffusionbasedreinforcementlearningqweighted}
Shutong Ding, Ke~Hu, Zhenhao Zhang, Kan Ren, Weinan Zhang, Jingyi Yu, Jingya Wang, and Ye~Shi.
\newblock Diffusion-based reinforcement learning via q-weighted variational policy optimization, 2024.

\bibitem{kang2023efficientdiffusionpoliciesoffline}
Bingyi Kang, Xiao Ma, Chao Du, Tianyu Pang, and Shuicheng Yan.
\newblock Efficient diffusion policies for offline reinforcement learning, 2023.

\bibitem{qgpo}
Cheng Lu, Huayu Chen, Jianfei Chen, Hang Su, Chongxuan Li, and Jun Zhu.
\newblock Contrastive energy prediction for exact energy-guided diffusion sampling in offline reinforcement learning, 2023.

\bibitem{qipo}
Shiyuan Zhang, Weitong Zhang, and Quanquan Gu.
\newblock Energy-weighted flow matching for offline reinforcement learning, 2025.

\bibitem{idql}
Philippe Hansen-Estruch, Ilya Kostrikov, Michael Janner, Jakub~Grudzien Kuba, and Sergey Levine.
\newblock Idql: Implicit q-learning as an actor-critic method with diffusion policies, 2023.

\bibitem{sfbc}
Huayu Chen, Cheng Lu, Chengyang Ying, Hang Su, and Jun Zhu.
\newblock Offline reinforcement learning via high-fidelity generative behavior modeling, 2023.

\bibitem{awr1}
Jan Peters and Stefan Schaal.
\newblock Reinforcement learning by reward-weighted regression for operational space control.
\newblock In {\em Proceedings of the 24th international conference on Machine learning}, pages 745--750, 2007.

\bibitem{awr2}
Qing Wang, Jiechao Xiong, Lei Han, Peng Sun, Han Liu, and Tong Zhang.
\newblock Exponentially weighted imitation learning for batched historical data.
\newblock In {\em Advances in Neural Information Processing Systems (NeurIPS)}, volume~31, pages 6291--6300, 2018.

\bibitem{awr}
Xue~Bin Peng, Aviral Kumar, Grace Zhang, and Sergey Levine.
\newblock Advantage-weighted regression: Simple and scalable off-policy reinforcement learning.
\newblock {\em CoRR}, abs/1910.00177, 2019.

\bibitem{d4rl}
Justin Fu, Aviral Kumar, Ofir Nachum, George Tucker, and Sergey Levine.
\newblock {D4RL:} datasets for deep data-driven reinforcement learning.
\newblock {\em CoRR}, abs/2004.07219, 2020.

\bibitem{florence2021implicitbehavioralcloning}
Pete Florence, Corey Lynch, Andy Zeng, Oscar Ramirez, Ayzaan Wahid, Laura Downs, Adrian Wong, Johnny Lee, Igor Mordatch, and Jonathan Tompson.
\newblock Implicit behavioral cloning, 2021.

\bibitem{li2022softmaxpolicygradientmethods}
Gen Li, Yuting Wei, Yuejie Chi, and Yuxin Chen.
\newblock Softmax policy gradient methods can take exponential time to converge, 2022.

\bibitem{mish}
Diganta Misra.
\newblock Mish: A self regularized non-monotonic activation function, 2020.

\bibitem{adam}
Diederik~P. Kingma and Jimmy Ba.
\newblock Adam: A method for stochastic optimization, 2017.

\bibitem{tsne}
Laurens van~der Maaten and Geoffrey Hinton.
\newblock Visualizing data using t-sne.
\newblock {\em Journal of Machine Learning Research}, 9(86):2579--2605, 2008.

\bibitem{sgd}
Sebastian Ruder.
\newblock An overview of gradient descent optimization algorithms, 2017.

\end{thebibliography}

\newpage
\appendix
\section{Theoretical analysis}\label{app:maxium_bound}
We proceed by induction on the value iteration steps.

\textbf{Base Case:}  
Initialize \(V_0(s) = 0\) (or any arbitrarily small value such that \(V_0(s) \leq V^*(s)\)). The base case holds trivially.

\textbf{Inductive Step:}  
Assume that at iteration \(k\), the value function satisfies \(V_k(s) \leq V^*(s)\) for all \(s \in S\). We examine the update for iteration \(k+1\).  
By setting the gradient of Equation~\ref{eq:value_loss} w.r.t $V_{\psi}$ to $0$, we get:
\begin{equation}
    V_{k+1}(s) = (1-p)\mathbb{E}^\tau_{a\sim D(s)}[Q_k(s,a)] + p\mathbb{E}_{\hat{a}\sim\pi_\phi(\hat{a}|s)}[Q_k(s,\hat{a})]
\end{equation}
where $\mathbb{E}^{\tau}$ denotes the $\tau$-expectile operator.
We analyze the two terms on the right-hand side separately.

\noindent\textbf{1. Expectile Term (In-Sample):}  
First, we bound the updated \(Q\)-values.
\begin{align}
Q_{k}(s,a) &= \mathbb{E}_{(r,s')\sim D(s,a)}\big[r + \gamma V_k(s')\big] \\
&\leq \mathbb{E}_{(r,s')\sim D(s,a)}\big[r + \gamma V^*(s')\big] \quad \text{(inductive hypothesis $V_k \leq V^*$)} \\
&\leq Q^*(s,a) \quad \text{(definition of $Q^*$)}
\end{align}

By Assumption ~\ref{ass:expectile-gap}, for any $s$ in dataset:
\begin{align}
    \mathbb{E}^\tau_{a\sim D(s)}[Q_k(s,a)] \leq \mathbb{E}^\tau_{a\sim D(s)}[Q^*(s,a)] \leq \max_{a \in D(s)} Q^*(s,a) - \delta_{\tau}
\end{align}

Invoking Assumption~\ref{ass:dataset-suboptimal}, the dataset support is suboptimal:
\begin{align}
    \max_{a \in D(s)} Q^*(s,a) \;\leq\; V^*(s) - (\delta_{\tau} + \delta_{sub}).
\end{align}

Thus, the first term, with probability greater than $\eta$, is bounded:
\begin{align}
   \mathbb{E}^\tau_{a\sim D(s)}[Q_{k}(s,a)] \;\leq\; V^*(s) - \delta_{sub} . 
\end{align}

\noindent\textbf{2. Policy Term (Out-of-Distribution):}  
For the second term, the policy \(\pi_\phi\) may query actions \(\hat{a}\) outside the dataset support. In practice \(Q\)-functions on OOD actions can suffer from overestimation due to function-approximation errors. We use the worst case bound:
\begin{align}
    \mathbb{E}_{\hat{a}\sim\pi_\phi}[Q_{k}(s,\hat{a})] \;\leq\; \frac{2R_{\max}}{(1-\gamma)} .
\end{align}
In practice, we can crop $Q$ if it ever exceeds the bound.

\vspace{0.5em}
\noindent\textbf{Combining the Bounds:}  
Substituting these bounds back into the update equation gives:
\begin{align}
    V_{k+1}(s) \;\leq\; (1-p)\big(V^*(s) - \delta_{\tau} - \delta_{sub}\big) \;+\; pV_{\max},
\end{align}
where $V_{\max}=\frac{2R_{\max}}{1-\gamma}$

To satisfy the inductive hypothesis \(V_{k+1}(s) \leq V^*(s)\), we require:
\begin{align}
    (1-p)\big(V^*(s) - (\delta_{\tau} + \delta_{sub})\big) + pV_{\max} \;\leq\; V^*(s).
\end{align}

Simplifying:

\begin{align}
p \leq \frac{\delta_{\tau} + \delta_{sub}}{V_{\max} - V^*(s) + (\delta_{\tau} + \delta_{sub})}.
\end{align}

Thus, if \(p\) is chosen according to this bound, the inductive step \(V_{k+1}(s) \leq Q^*(s)\) holds. By induction, the safety guarantee holds for the fixed point \(\hat{V}\).$\square$

%%%%%%%%%%%%%%%%%%%%%%%%%%%%%%%%%%%%%%%%%%%%%%%%%%%%%%

\section{Theoretical Analysis of Stochastic vs. Deterministic Interpolation of Policy}
\label{app:mode_collapse_proof}
We analyze the gradient dynamics of ISEP's stochastic action selection strategy compared to deterministic action interpolation. We consider a simplified case where the policy $\pi_\phi(a|s)$ is parameterized as a Gaussian $\mathcal{N}(\mu_\phi(s), \Sigma)$. Let the dataset action be $a_{\mathcal{D}}$ and the optimistic policy action be $a_{\pi}$.
\paragraph{Case 1: Deterministic Action Interpolation}\label{app:case1}
To analyze the gradient dynamics of deterministic action interpolation, we consider a simplified policy optimization setting where the policy $\pi_\phi(a \mid s)$ is parameterized as a Gaussian distribution $\mathcal{N}(\mu_\phi(s), \sigma^2 I)$ with fixed isotropic covariance $\sigma^2 I$. Here, $\mu_\phi(s)$ denotes the state-dependent mean parameterized by $\phi$, the dataset action is denoted $a_{\mathcal{D}}$, and the action sampled from current policy is denoted $a_{\pi}$. We assume the objective is to minimize a weighted negative log-likelihood loss, which is equivalent to maximizing the advantage-weighted log-probability:
\begin{align}
\mathcal{L'}_{\pi}(\phi) = p \cdot \mathbb{E}_{(s, a_{\mathcal{D}}) \sim \mathcal{D}} \bigl[ \omega(s, a_{\mathcal{D}}) \log \pi_\phi(a_{\mathcal{D}} \mid s) \bigr] + (1-p) \cdot \mathbb{E}_{\substack{s \sim \mathcal{D}, \\ a_{\pi} \sim \pi_\phi(\cdot \mid s)}} \bigl[ \omega(s, a_{\pi}) \log \pi_\phi(a_{\pi} \mid s) \bigr],
\label{eq:det_loss}
\end{align}
where $p \in [0,1]$ balances the off-policy (dataset) and on-policy contributions, and $\omega(s, a)$ are fixed scalar weights (e.g., advantage estimates) that depend on the action but are detached during backpropagation (i.e., treated as constants with respect to $\phi$).

Since the loss decomposes additively over states $s$ (assuming $\mathcal{D}$ induces a product distribution over state-action pairs), we can derive the critical point of $\mathcal{L'}_\pi(\phi)$ on a per-state basis. For a fixed state $s$, let $\mathcal{D}_s$ denote the conditional dataset distribution over actions $a_{\mathcal{D}} \sim \mathcal{D}_s$, and let the expectations be taken accordingly (marginalizing over $s \sim \mathcal{D}$ yields the full loss gradient). The log-probability under the Gaussian policy is $\log \pi_\phi(a \mid s) = -\frac{1}{2\sigma^2} \|a - \mu_\phi(s)\|^2 + c$, where $c$ is a constant independent of $\phi$. Differentiating with respect to $\mu_\phi(s)$ (treating $\sigma^2$ as fixed and using the detachment of $\omega$) yield:
\begin{align}
\nabla_{\mu_\phi(s)} \mathcal{L'}_\pi(\phi) = \frac{1}{\sigma^2} \Biggl( p \cdot \mathbb{E}_{a_{\mathcal{D}} \sim \mathcal{D}_s} \bigl[ \omega(s, a_{\mathcal{D}}) (\mu_\phi(s) - a_{\mathcal{D}}) \bigr] + (1-p) \cdot \mathbb{E}_{a_{\pi} \sim \pi_\phi(\cdot \mid s)} \bigl[ \omega(s, a_{\pi}) (\mu_\phi(s) - a_{\pi}) \bigr] \Biggr).
\label{eq:grad_det}
\end{align}
Setting this gradient to zero (the first-order necessary condition for a minimum, assuming convexity in $\mu_\phi(s)$) and solving for $\mu_\phi(s)$ gives the fixed-point equation:
\begin{align}
\mu_\phi(s) &= \frac{ p \cdot \mathbb{E}_{a_{\mathcal{D}} \sim \mathcal{D}_s} \bigl[ \omega(s, a_{\mathcal{D}}) a_{\mathcal{D}} \bigr] + (1-p) \cdot \mathbb{E}_{a_{\pi} \sim \pi_\phi(\cdot \mid s)} \bigl[ \omega(s, a_{\pi}) a_{\pi} \bigr] }{ p \cdot \mathbb{E}_{a_{\mathcal{D}} \sim \mathcal{D}_s} \bigl[ \omega(s, a_{\mathcal{D}}) \bigr] + (1-p) \cdot \mathbb{E}_{a_{\pi} \sim \pi_\phi(\cdot \mid s)} \bigl[ \omega(s, a_{\pi}) \bigr] }.
\label{eq:mu_interp}
\end{align}
This equation defines an implicit fixed point for $\mu_\phi(s)$, as the second expectation depends on the current policy distribution $\pi_\phi(\cdot \mid s)$. In practice, it can be solved iteratively (e.g., via fixed-point iteration or as part of gradient-based optimization).

Equation~\eqref{eq:mu_interp} reveals that the optimal policy mean $\mu_\phi(s)$ is a linear interpolation between the weighted expected dataset action $\mathbb{E}[\omega(s, a_{\mathcal{D}}) a_{\mathcal{D}}]/\mathbb{E}[\omega(s, a_{\mathcal{D}})]$ and the weighted expected policy action $\mathbb{E}[\omega(s, a_{\pi}) a_{\pi}]/\mathbb{E}[\omega(s, a_{\pi})]$, with global mixing weights $p$ and $1-p$. 
For mini-batch training, the expectations can be estimated via Monte Carlo samples, reducing to a pointwise interpolation:
\begin{align}
\tilde{\mu}_\phi(s) \approx \frac{ p \, \omega(s, a_{\mathcal{D}}) \, a_{\mathcal{D}} + (1-p) \, \omega(s, a_{\pi}) \, a_{\pi} }{ p \, \omega(s, a_{\mathcal{D}}) + (1-p) \, \omega(s, a_{\pi}) }.
\end{align}
where $a_{\mathcal{D}} \sim \mathcal{D}_s$ and $a_{\pi} \sim \pi_\phi(\cdot \mid s)$ are single samples; repeated applications converge to the expectation-based solution in \eqref{eq:mu_interp} under standard ergodicity assumptions.

This deterministic interpolation provides a stable baseline for policy updates, as it directly blends empirical data and policy samples without additional stochasticity beyond sampling. In contrast, ISEP's stochastic action selection (analyzed in Case 2) introduces variance through randomized mixing, which we compare empirically in Section~\ref{sec:ablation_action_mix}.

\paragraph{Case 2: Stochastic Mixture Strategy (Action Selection)}
In the stochastic setting, ISEP decouples the objectives using a Bernoulli gate $B \sim \text{Bernoulli}(p)$. The loss function for a single optimization step is defined as:
\begin{align}
\mathcal{L}_{\pi}(\phi) = B \cdot \mathbb{E}_{(s,a) \sim \mathcal{D}}[\omega(s,a)\log \pi_\phi(a \mid s)]
+ (1-B) \cdot \mathbb{E}_{\substack{s \sim \mathcal{D},\ \hat{a} \sim \pi_{\phi}(\cdot|s)}}[\omega(s,\hat{a})\log \pi_\phi(\hat{a} \mid s)]
\label{eq:stoch_loss}
\end{align}
where $B \in \{0, 1\}$ is sampled at each iteration and is detached from the computation graph.

To analyze the update dynamics, we derive the gradient of Equation~\ref{eq:stoch_loss} with respect to the Gaussian mean parameter $\mu_\phi(s)$:
\begin{align}
\nabla_{\mu_\phi} \mathcal{L}_{\pi}(\phi) = \frac{1}{\sigma^2} \left( B \cdot \mathbb{E}_{\mathcal{D}}\Big[\omega(s,a)(a - \mu_\phi)\Big] + (1-B) \cdot \mathbb{E}_{\pi}\Big[\omega(s,\hat{a})(\hat{a} - \mu_\phi)\Big] \right)
\label{eq:stoch_grad}
\end{align}

\textbf{Consistency in Expectation:}
Taking the expectation of the gradient over the Bernoulli distribution of $B$ (where $\mathbb{E}[B]=p$):
\begin{align}
\mathbb{E}_B \left[ \nabla_{\mu_\phi} \mathcal{L}_{\pi}(\phi) \right] &= \frac{1}{\sigma^2} \left( p \mathbb{E}_{\mathcal{D}}\Big[\omega(s,a)(a - \mu_\phi)\Big] + (1-p) \mathbb{E}_{\pi}\Big[\omega(s,\hat{a})(\hat{a} - \mu_\phi)\Big] \right)
\end{align}
This result is identical to Equation (2) in Case 1. Thus, the stochastic objective shares the same global fixed point $\mu^*_\phi$ as the deterministic interpolation in the limit of infinite samples.

\textbf{Instantaneous Gradient Dynamics (Action Selection):}
While the expectation is identical, the instantaneous descent direction differs fundamentally. At any specific training step $t$, the realization of $B_t$ creates two disjoint update regimes. Let the effective target action vector $y_t$ be defined based on the realization of $B_t$:
\begin{align}
\nabla_{\mu_\phi} \mathcal{L}_{\pi}(\phi) \propto \begin{cases}
\mathbb{E}_{\mathcal{D}}[\omega(s,a)(a - \mu_\phi)] & \text{if } B_t = 1 \implies \text{Target: } y_t \leftarrow a \sim \mathcal{D} \\
\mathbb{E}_{\pi}[\omega(s,\hat{a})(\hat{a} - \mu_\phi)] & \text{if } B_t = 0 \implies \text{Target: } y_t \leftarrow \hat{a} \sim \pi
\end{cases}
\end{align}

Unlike Case 1, where the gradient vector points towards a linear combination (interpolation) of $a$ and $\hat{a}$, the gradient in Case 2 forces the policy mean to move strictly towards \textit{either} the dataset distribution \textit{or} the current policy distribution at any given step. This prevents the optimization trajectory from traversing regions of the parameter space that correspond to invalid interpolations of actions, effectively acting as a stochastic switch rather than a blender.

\paragraph{Remark: Optimization Trajectory and Validity.}
It is important to note that while $\mathbb{E}[\nabla_{\text{stoch}}] = \nabla_{\text{det}}$, the optimization trajectories differ significantly due to the variance of the gradient estimator. In the deterministic case (Eq. \ref{eq:det_loss}), the gradient vector at every step points toward the linear combination of the targets. If the valid action space is non-convex (e.g., an obstacle exists between $a_{\mathcal{D}}$ and $a_{\pi}$), the deterministic gradient may consistently push the policy parameters $\phi$ through regions of high loss (invalid actions). In contrast, the stochastic mixture (Eq. \ref{eq:policy_loss}) performs \textit{action selection}. By sampling the objective gate $B$ (particularly if $B$ is fixed per optimization step), the gradient update is directed entirely toward either the dataset manifold or the current policy manifold. The optimization trajectory effectively "zig-zags" between valid regions rather than traversing the invalid interpolation zone. This stochasticity also introduces a regularizing noise term proportional to the disagreement between the dataset and the policy:
\[
\text{Var}(\nabla \mathcal{L}_{\text{stoch}}) = \text{Var}(\nabla \mathcal{L}_{\text{det}}) + p(1-p) \Big\| \nabla \mathcal{L}_{\mathcal{D}} - \nabla \mathcal{L}_{\pi} \Big\|^2
\]
This variance term is non-zero whenever the gradients disagree, discouraging the optimization from settling in unstable interpolated regions. Besides, this noise is beneficial: it enables escaping poor local minima while acting as implicit regularization---the system naturally avoids regions with conflicting gradient signals~\cite{sgd}.

\textbf{Remark on Batch Dynamics:} Even when $B$ is sampled element-wise within a minibatch, the stochastic objective ensures that the loss is computed pointwise against valid targets. The gradient contribution for the $i$-th sample is $\nabla \mathcal{L}_i \propto (\mu_\phi(s_i) - y_i)$, where $y_i \in \{a_{\mathcal{D}}, a_{\pi}\}$. This avoids the explicit construction of potentially invalid target vectors.

\section{Hyperparameter Details and Reproducibility}
\label{sec:appendix_hyperparams}
To ensure the reproducibility of our experimental results, we provide the specific task-dependent hyperparameters used for both variants of our algorithm. Table \ref{tab:hyperparams_base} lists the parameters for the Gaussian variant (ISEP), while Table \ref{tab:hyperparams_fm} details the parameters for the Flow-Matching variant (ISEP-FM).
\begin{table}[H]
\centering
\caption{Hyperparameters for the base Gaussian ISEP across D4RL tasks. The table details the interpolation parameter ($p$), expectile ($\tau$), and the sampling temperature ($\beta$).}
\label{tab:hyperparams_base}
\begin{tabular}{lccc}
\toprule
\textbf{Task} & \textbf{$p$ } & \textbf{$\tau$ (expectile)} & \textbf{$\beta$ (temperature)} \\
\midrule
halfcheetah-medium-v2          & 0.5 & 0.8 & 5.0 \\
hopper-medium-v2               & 0.5 & 0.8 & 3.0 \\
walker2d-medium-v2             & 0.3 & 0.8 & 5.0 \\
halfcheetah-medium-replay-v2   & 0.5 & 0.8 & 5.0 \\
hopper-medium-replay-v2        & 0.3 & 0.8 & 3.0 \\
walker2d-medium-replay-v2      & 0.5 & 0.8 & 5.0 \\
halfcheetah-medium-expert-v2   & 0.3 & 0.8 & 3.0 \\
hopper-medium-expert-v2        & 0.3 & 0.8 & 3.0 \\
walker2d-medium-expert-v2      & 0.3 & 0.8 & 3.0 \\
pen-human-v1                   & 0.2 & 0.7 & 0.5 \\
pen-cloned-v1                  & 0.5 & 0.7 & 0.5 \\
kitchen-complete-v0            & 0.5 & 0.7 & 0.05 \\
kitchen-partial-v0             & 0.2 & 0.8 & 0.1 \\
kitchen-mixed-v0               & 0.3 & 0.8 & 0.1 \\
\bottomrule
\end{tabular}
\end{table}

\begin{table}[H]
\centering
\caption{Hyperparameters for the ISEP-FM variant across D4RL tasks. The \textbf{guidance\_weight} ($w$) denotes the classifier-free guidance weight applied during inference. Analogous to the temperature ($\beta$) in base ISEP, this guidance weight controls the degree of policy improvement.}
\label{tab:hyperparams_fm}

\begin{tabular}{lcccc}
\toprule
\textbf{Task} & \textbf{$p$} & \textbf{$\tau$ (expectile)} & \textbf{$w$ (guidance\_weight)} \\
\midrule
halfcheetah-medium-v2          & 0.3 & 0.7 & 2.0  \\
hopper-medium-v2               & 0.2 & 0.7 & 1.0  \\
walker2d-medium-v2             & 0.3 & 0.8 & 2.0  \\
halfcheetah-medium-replay-v2   & 0.3 & 0.8 & 2.0  \\
hopper-medium-replay-v2        & 0.5 & 0.8 & 3.0 \\
walker2d-medium-replay-v2      & 0.3 & 0.8 & 3.0 \\
halfcheetah-medium-expert-v2   & 0.3 & 0.8 & 3.0 \\
hopper-medium-expert-v2        & 0.2 & 0.7 & 3.0 \\
walker2d-medium-expert-v2      & 0.3 & 0.7 & 2.0 \\
pen-human-v1                   & 0.2 & 0.7 & 3.0 \\
pen-cloned-v1                  & 0.2 & 0.7 & 1.0 \\
kitchen-complete-v0            & 0.2 & 0.8 & 0.5 \\
kitchen-partial-v0             & 0.2 & 0.8 & 1.0 \\
kitchen-mixed-v0               & 0.2 & 0.8 & 0.2 \\
\bottomrule
\end{tabular}
\end{table}

\end{document}